\documentclass[runningheads]{llncs}

\usepackage[T1]{fontenc}
\pdfoutput=1

\makeatletter
\renewcommand\subsubsection{\@startsection{subsubsection}{3}{\z@}%
                       {-18\p@ \@plus -4\p@ \@minus -4\p@}%
                       {0.5em \@plus 0.22em \@minus 0.1em}%
                       {\normalfont\normalsize\bfseries\boldmath}}
\makeatother
\usepackage{amsmath}
\usepackage{amssymb}
\usepackage{mathtools}
\usepackage{thm-restate}
\usepackage{enumerate}
\usepackage[shortlabels]{enumitem}
\usepackage{xcolor}
\usepackage{proof}
\usepackage{array}
\usepackage{xspace}
\usepackage{algorithm}
\usepackage[noend]{algpseudocode}
\usepackage{xparse}
\usepackage{hyperref}
\usepackage{tikz,pgfplots}
\usetikzlibrary{automata, arrows.meta, positioning,decorations.markings,shapes.misc}
\usepackage{xfrac}
\usepackage{wrapfig}
\usepackage{caption,subcaption}
\usepackage[croatian,english]{babel}
\usepackage{stmaryrd}
\usepackage{graphicx}

\setlist[description]{leftmargin=1cm,labelindent=0.5cm}

\newcommand{\stam}[1]{}

\newcommand{\tup}[1]{\ensuremath{\left\langle #1 \right\rangle}}
\newcommand{\set}[1]{\ensuremath{\left\lbrace #1 \right\rbrace}}

\newcommand{\proj}[2]{\ensuremath{\mathsf{proj}_{#1}(#2)}}
\newcommand{\rec}{\ensuremath{\mathit{Inf}}}
\newcommand{\pref}{\ensuremath{\mathsf{pref}}}
\newcommand{\cmpl}[1]{\ensuremath{#1^c}}

\newcommand{\zug}[1]{\langle #1 \rangle}

\mathchardef\mhyphen="2D

\newcommand{\strongprob}{\mathsf{STRONG\mhyphen SYNT}}
\newcommand{\aaprob}{\mathsf{AA\mhyphen SYNT}}
\newcommand{\agprob}{\mathsf{AG\mhyphen SYNT}}

\newcommand{\always}{\ensuremath{\mathbf{G}}}

\newcommand{\T}{\ensuremath{\mathcal{G}}}
\newcommand{\G}{\ensuremath{\mathcal{G}}}

\newcommand{\V}{\ensuremath{V}}

\newcommand{\E}{\ensuremath{E}}

\newcommand{\vinit}{\ensuremath{v^{0}}}

\NewDocumentCommand{\pol}{gg}{\ensuremath{\pi}\IfNoValueTF{#1}{}{(#1,#2)}}
\newcommand{\actpol}{\alpha} 

\newcommand{\bidpol}{\beta} 

\newcommand{\infpaths}{\ensuremath{\mathit{Paths}^{\mathsf{inf}}}}
\newcommand{\finpaths}{\ensuremath{\mathit{Paths}^{\mathsf{fin}}}}
\renewcommand{\path}{\ensuremath{\mathit{out}}}
\newcommand{\spec}{\ensuremath{\Phi}}
\NewDocumentCommand{\last}{g}{\ensuremath{\textit{last}\IfNoValueTF{#1}{(h)}{(#1)}}}
\renewcommand{\S}{\mathcal{S}}

\renewcommand{\H}{\mathcal{H}}
\newcommand{\C}{\mathcal{C}}
\newcommand{\B}{\mathbb{B}}
\renewcommand{\t}{\tau}

\newcommand{\out}{\mathrm{out}}

\NewDocumentCommand{\spoiling}{gg}{\ensuremath{\textit{spoil}_{\IfNoValueTF{#2}{}{#2}}(#1)}} 
\NewDocumentCommand{\helping}{gg}{\ensuremath{\textit{help}_{\IfNoValueTF{#2}{}{#2}}(#1)}} 

\newcommand{\comp}[2]{\ensuremath{#1\!\bowtie\! #2}}
\newcommand{\compOT}{\comp{\t_1}{\t_2}}

\newcommand{\A}{\ensuremath{\mathcal{A}}}
\renewcommand{\P}{\ensuremath{\mathit{Player}}}
\newcommand{\PO}{\ensuremath{\mathit{\P~X}}\xspace}
\newcommand{\PT}{\ensuremath{\mathit{\P~Y}}\xspace}
\newcommand{\PLi}{\ensuremath{\mathit{\P~i}}\xspace}

\NewDocumentCommand{\Adm}{gg}{\textit{Adm}\IfNoValueTF{#1}{}{
		\IfNoValueTF{#2}{(#1)}{^{#2}(#1)}}}

\NewDocumentCommand{\arch}{gggg}{\ensuremath{\IfNoValueTF{#1}{\A}{\tup{#1, #2, #3, #4}}}}

\NewDocumentCommand{\trimmed}{g}{\ensuremath{\IfNoValueTF{#1}{\T}{#1}_{\mathsf{adm}}}~}

\NewDocumentCommand{\las}{ggg}{\widehat{\IfNoValueTF{#1}{\G}{\G}}_{\IfNoValueTF{#2}{}{#2,} \IfNoValueTF{#3}{}{#3}}}
\newcommand{\lasb}[1]{\widehat{\G}_\buchiset}
\newcommand{\head}[1]{\bar{\G}}

\NewDocumentCommand{\reach}{g}{\ensuremath{\mathit{Reach}\IfNoValueTF{#1}{}{(#1)}}}
\newcommand{\safe}{\ensuremath{\mathit{Safe}}}
\newcommand{\buchiword}{B\"uchi~}
\newcommand{\buchi}{\ensuremath{\mathit{B\ddot{u}chi}}}
\newcommand{\buchiset}{\ensuremath{\mathcal{B}}}

\newcommand{\sat}{\ensuremath{\models}}

\newcommand{\parity}{\ensuremath{\mathit{Parity}}}
\newcommand{\col}{\ensuremath{\mathit{Col}}}

\renewcommand{\th}{\ensuremath{\mathit{Th}}}
\NewDocumentCommand{\thresh}{ggg}{\th_{#2}^{\IfNoValueTF{#3}{}{#3}}(#1)} 

\NewDocumentCommand{\thadmo}{mgg}{\ensuremath{\mathit{Th}^{\mathsf{adm}_{\IfNoValueTF{#3}{\spec_2}{#3}}}_{\IfNoValueTF{#2}{\spec_1}{#2}}(#1)}}
\NewDocumentCommand{\thadmt}{mgg}{\ensuremath{\mathit{Th}^{\mathsf{adm}_{\IfNoValueTF{#2}{\spec_1}{#2}}}_{\IfNoValueTF{#3}{\spec_2}{#3}}(#1)}}
\NewDocumentCommand{\thadmi}{mgg}{\ensuremath{\mathit{Th}^{\mathsf{adm}_{\IfNoValueTF{#2}{\spec_{3-i}}{#2}}}_{\IfNoValueTF{#3}{\spec_i}{#3}}(#1)}}

\newcommand{\Tenders}{\mathcal{T}}

\NewDocumentCommand{\config}{gg}{\tup{\ensuremath{\IfNoValueTF{#1}{v}{#1}, \IfNoValueTF{#2}{B}{#2}}}}

\NewDocumentCommand{\stratresstrict}{gg}{\IfNoValueTF{#1}{\sigma}{#1}|_{\IfNoValueTF{#2}{h}{#2}}}
\NewDocumentCommand{\optrange}{g}{\ensuremath{\textsf{Opt}(\IfNoValueTF{#1}{B}{#1})}} 
\newcommand{\Ass}{A}

\newcommand{\AG}[3]{\langle #1 \triangleright #2 \triangleright #3\rangle}

\NewDocumentCommand{\func}{ggg}{f_{#2}^{\IfNoValueTF{#3}{}{#3}}(#1)} 
\newcommand{\red}{\mathsf{red}}
\newcommand{\blue}{\mathsf{blue}}

\spnewtheorem{myproblem}{Problem}{\bfseries}{\itshape}

\def\orcid#1{\kern .08em\href{https://orcid.org/#1}{\includegraphics[keepaspectratio,width=0.7em]{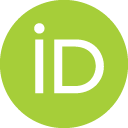}}}

\begin{document}

\title{Auction-Based Scheduling}
\author{Guy Avni\inst{1}\orcid{0000-0001-5588-8287} \and
		Kaushik Mallik\inst{2}\orcid{0000-0001-9864-7475} \and
		Suman Sadhukhan\inst{1}\orcid{0000-0002-4802-6803}}
			
\authorrunning{Avni et al.}

\institute{University of Haifa, Haifa, Israel,\\
			\email{gavni@cs.haifa.ac.il},\\
			\email{ssadhukh@campus.haifa.ac.il}\\
			\and
			Institute of Science and Technology Austria (ISTA), Klosterneuburg, Austria,\\
			\email{kaushik.mallik@ist.ac.at}
}

	\maketitle
	
	\begin{abstract}
Sequential decision-making tasks often require satisfaction of multiple, partially-contradictory objectives.
		Existing approaches are monolithic, where a single {\em policy} 
		fulfills all objectives.
		We present \emph{auction-based scheduling}, a \emph{decentralized} framework for multi-objective sequential decision making. 
		Each objective is fulfilled using a separate and independent policy. Composition of policies is performed at runtime, where at each step, the policies simultaneously bid from pre-allocated budgets for the privilege of choosing the next action. The framework allows policies to be independently created, modified, and replaced. 
We study path planning problems on finite graphs with two temporal objectives and present algorithms to synthesize policies together with bidding policies in a  decentralized manner. 
We consider three categories of decentralized synthesis problems, parameterized by the assumptions that the policies make on each other. We identify a class of assumptions called {\em assume-admissible} for which synthesis is always possible for graphs whose every vertex has at most two outgoing edges. 
	\end{abstract}


\section{Introduction}

Sequential decision-making tasks often require satisfaction of multiple, partially-contradictory objectives.
For example, the control {\em policy} of a traffic light may need to choose signals in a way that the traffic throughput is maximized \emph{while} the maximum waiting time is minimized~\cite{houli2010multiobjective}, 
the control policy operating an unmanned aerial vehicle may need to navigate in a way that the destination is reached \emph{while} no-fly zones are avoided~\cite{hahn2023multi},
the policy of an operating-system resource manager needs to allocate resources to tasks in a way that deadlocks are avoided \emph{while} fairness is maintained~\cite{de2005code}.

We propose a decentralized synthesis framework for policies when tasks are given as a conjunction of two objectives $\Phi_1$ and $\Phi_2$, and the policies need to choose actions from a common action space. 
The key idea is that $\spec_1$ and $\spec_2$ will be accomplished, respectively, using two {\em action policies} $\actpol_1$ and $\actpol_2$---designed independently, and the composition of $\actpol_1$ and $\actpol_2$ at runtime will generate a policy for $\Phi_1 \wedge \Phi_2$. 
The challenge is that at each time point, one action needs to be chosen, whereas $\actpol_1$ and $\actpol_2$ might select conflicting actions. 
For example, when developing a plan for a robot, $\Phi_1$ and $\Phi_2$ might specify two target locations, and $\actpol_1$ and $\actpol_2$ may select opposite directions in a location. 

We propose a novel composition mechanism called \emph{auction-based scheduling}: 
both policies are allocated bounded monetary \emph{budgets}, and at each point in time, an auction (aka {\em bidding}) is held, where the policies {\em bid} from their budgets for the privilege to get scheduled for choosing the action. More formally, we equip each action policy $\actpol_i$, for $i \in \set{1,2}$, with a {\em bidding policy} $\bidpol_i$, which is a function that proposes a bid from the available budget based on the history of the interaction. 
A {\em tender} for objective $\Psi$ is a triple  $\t = \zug{\actpol, \bidpol, \B}$, where $\actpol$ is an {\em action policy}, $\bidpol$ is a {\em bidding policy}, and $\B \in (0,1)$ is a minimal budget required for the tender to guarantee $\Psi$. Two tenders $\t_1$ and $\t_2$ are {\em compatible} if $\B_1 + \B_2 < 1$, which is when they can be composed at runtime as follows. 
Each Tender~$i$, for $i \in \set{1,2}$, is allocated an initial budget that exceeds $\B_i$, where the sum of budgets equals $1$. At each point in time, the tenders simultaneously choose bids using their bidding policies, the higher bidder chooses an action using its action policy and pays the bid to the other tender. Thus, the sum of budgets stays constant at $1$. Note that the composition gives rise to a path in the graph. 
The decentralized synthesis problem asks:
Given a graph $\G$ and objectives $\spec_1,\spec_2$ such that $\spec_1\land\spec_2\neq \mathtt{false}$, for each $\spec_i$ compute $\t_i$ such that no matter which tender it is composed with, the composition generates a path that fulfills $\spec_i$. The framework is sound-by-construction, namely the composition of compatible tenders satisfies $\spec_1\land\spec_2$.

The advantage of auction-based scheduling is modularity at two levels.
First, since the designs of policies do not depend on each other, they can be created independently and in parallel, e.g., by different vendors or in a parallel computation.
Second, since the policies operate independently, they can be modified and replaced separately. For example, when only the objective $\Phi_1$ changes, there is no need to alter the policy $\actpol_2$, and vice versa. 

Bidding for the next action encourages the policy with higher scheduling urgency to bid higher, and at the same time, the bounds on budgets 
ensure {\em fairness}, namely that no policy is starved. 
Auction-based scheduling adds new, complementary features to the arsenal of modular approaches in multi-objective decision-making.
With the conventional decentralized synthesis approaches, the policies are composed either \emph{concurrently}~\cite{majumdar2020assume} or in a \emph{turn-based} manner~\cite{brenguier2015assume}.
Concurrent actions are meaningful if each policy needs to act on its own local control variables, e.g., when the local control policies of two robots concurrently move the robot towards their destinations in a shared workspace.
In our case, the set of actions is common between policies, and the concurrent interaction is unsuitable.
Likewise, turn-based actions are also unsuitable in our setting because it is unclear how to assign turns to policies apriori.
We will demonstrate (Ex.~\ref{ex:turn-based fails}) that an inappropriate turn-assignment to policies may violate some of the objectives, while auction-based scheduling will succeed to fulfil all of them.

\stam{
We introduce \emph{tenders} which extend regular path planning policies with bidding abilities.
Each tender requires a minimum amount of initial budget share---called the \emph{threshold budget}---to be able to fulfill its objective; we normalize the total budget share to $1$, and therefore the threshold budget is a constant in $[0,1]$.
A pair of tenders $\t_1$ and $\t_2$ are composed using auction-based scheduling that works as follows.
The tenders start at a designated initial vertex in $\G$ with initial budgets above their thresholds.
At each step, the tenders simultaneously and independently make bids not exceeding their current budgets.
Whoever bids higher pays the bid amount to the other tender from its budget, and chooses the next edge for extending the current path.
The process continues and we obtain an infinite path over the graph.

The crucial criterion for composition is \emph{compatibility} between tenders, which requires that the sum of the threshold budgets of $\t_1$ and $\t_2$ is less than the total available budget $1$.
The decentralized synthesis problem asks:
Given a graph $\G$ and objectives $\spec_1,\spec_2$ such that $\spec_1\land\spec_2\neq \mathtt{false}$, for each $\spec_i$ compute $\t_i$ such that no matter what $\t_{3-i}$ does, their composition generates a path that fulfills $\spec_i$.
If $\t_1$ and $\t_2$ obtained this way are also compatible, then their composition will fulfill $\spec_1\land\spec_2$.
Sometimes, no compatible pair of tenders can be found, unless the tenders make some assumptions on each other.
}

We study auction-based scheduling in the context of path planning on finite directed graphs with pairs of $\omega$-regular objectives on its paths, and present algorithms for the decentralized synthesis problem with increasing levels of assumptions made by the tenders on each other:%
~(a) Strong synthesis, with no assumptions and the most robust solution, (b)~assume-admissible synthesis, with the assumption that the other tender is not purely cynical and behaves rationally with respect to its own objective, and (c)~assume-guarantee synthesis, with explicit contract-based pre-coordination.
We show 	for graphs whose every vertex has at most two outgoing edges, for every pair of $\omega$-regular objectives $\spec_1,\spec_2$, and for all three classes of problems (a), (b), and (c), there exist PTIME decentralized synthesis algorithms that either compute compatible tenders or output that no compatible tenders with the respective assumptions exist; surprisingly, we show that compatible tenders always exist for (b).
For general graphs, we show that the problems are in $\text{NP}\cap\text{coNP}$. All our algorithms internally solve \emph{bidding games} using known algorithms from the literature~\cite{LLPU96,LLPSU99}.
Due to the lack of space, some proofs are omitted, but can be found in the extended version~\cite{AMS23}. 

\section{Preliminaries}

Let $\Sigma$ be a finite alphabet. We use $\Sigma^*$ and $\Sigma^\omega$ to respectively denote the set of finite and infinite words over $\Sigma$, and $\Sigma^\infty$ to denote $\Sigma^*\cup \Sigma^\omega$.
Let for two words $u\in\Sigma^*$ and $v\in\Sigma^\infty$, $u\leq v$ denote that $u$ is a prefix of $v$, i.e., there exists a $w$ such that $v = uw$.
Given a language $L\subseteq \Sigma^\infty$, define $\pref(L)$ to be the set of every finite prefix in $L$, i.e., $\pref(L)\coloneqq \set{u\in \Sigma^*\mid \exists v\in L\;.\; u\leq v}$.

\paragraph*{Graphs.}
We formalize path planning problems on \emph{graphs}.
A graph $\G$ is a tuple $\tup{\V,\vinit,\E}$ where $\V$ is a finite set of vertices, $\vinit$ is a designated initial vertex, and $E\subseteq V\times V$ is a set of directed edges.
If $(u,v)\in \E$, then $v$ is a \emph{successor} of $u$.
A \emph{binary graph} is a graph whose every vertex has at most two successors.
A \emph{path} over $\T$ is a sequence of vertices $v^0v^1\ldots$ so that every $(v^i,v^{i+1})\in \E$.
Unless explicitly mentioned, paths always start at $v^0$.
We use $\finpaths(\T)$ and $\infpaths(\T)$ to denote the sets of finite and infinite paths, respectively.

A \emph{strongly connected component} (SCC) of the graph $\T$ is a set $S$ of vertices, such that there is a path between every pair of vertices of $S$.
An SCC $S$ is called a \emph{bottom strongly connected component} (BSCC) if there does not exist any edge from a vertex in $S$ to a vertex outside of $S$.
The graph $\T$ is itself called strongly connected if $\V$ is an SCC.

\paragraph*{Objectives.} 
Fix a graph $\G$.
An \emph{objective} $\spec$ \emph{in} $\G$ is a set of infinite paths, i.e., $\spec\subseteq \infpaths(\T)$.
For an infinite path $\rho$, we use $\rec(\rho)$ to denote the set of vertices that $\rho$ visits infinitely often.
We focus on the following objectives:
\begin{description}
	\item[Reachability:] for $S\subseteq \V$, $\reach_\T(S)\coloneqq \set{\vinit v^1\ldots \in \infpaths(\T) \mid \exists i\geq 0\; . \; v^i\in S}$,
	\item[Safety:] for $S\subseteq \V$, $\safe_\T(S)\coloneqq \set{\vinit v^1\ldots \in \infpaths(\T) \mid \forall i\geq 0\; . \; v^i\in S}$,
	\item[B\"uchi:] for $S\subseteq \V$, $\buchi_\T(S)\coloneqq \set{\rho \in \infpaths(\T)\mid \rec(\rho)\cap S\neq \emptyset}$,
	\item[Parity (max, even):] for $\col\colon \V\to [0;k]$ for some $k>0$, $\parity_\G(\col)\coloneqq \set{\rho\in\infpaths(\G)\mid \max\set{i\mid \exists v\in \rec(\rho)\;.\;\col(v)=i} \text{ is even}}$,
\end{description}
Given an objective $\spec$, we will use $\cmpl{\spec}$ to denote its complement, i.e., $\cmpl{\spec} = \infpaths(\G)\setminus \spec$.
Observe that $\cmpl{(\reach_\T(S))} = \safe_\T(V\setminus S)$.

\paragraph*{Action policies.}
Fix a graph $\T$.
An \emph{action policy} is a function $\actpol\colon \finpaths(\T)\to \V$, choosing the next vertex to extend any given finite path $\rho v$, where $\tup{v, \actpol(\rho v)} \in E$.
The action policy $\actpol$ is \emph{memoryless} if for every pair of distinct finite paths $\rho v,\rho'v$ that end in the same vertex $v$, it holds that $\actpol(\rho v) = \actpol(\rho' v)$; in this case we simply write $\actpol(v)$.
An action policy $\actpol$ generates a unique infinite path in $\T$, denoted $\path(\actpol)$, and defined inductively as follows. The initial vertex is $\vinit$. For every prefix $\vinit, \ldots, v^i$ of $\path(\actpol)$, for $i \geq 0$, $v^{i+1} = \actpol(\vinit, \ldots, v^i)$. 
We say that the policy $\actpol$ \emph{satisfies} a given objective $\spec$, written $\actpol\models \spec$ iff $\path(\actpol)\in \spec$.


\section{The Auction-Based Scheduling Framework}
\label{sec:auction-based scheduling}

Consider a graph $\G=\tup{\V,\vinit,\E}$.
A pair of objectives $\spec_1,\spec_2 \subseteq V^\omega$ in $\G$ are called \emph{overlapping} if they have nonempty intersection (i.e., $\spec_1 \cap \spec_2 \neq \emptyset$).
The \emph{multi-objective planning problem} asks to synthesize an action policy that satisfies the \emph{global} objective $\spec_1 \cap \spec_2$ for overlapping $\spec_1,\spec_2$. 

We propose a decentralized approach to the problem. 
Our goal is to design two action policies $\actpol_1$ and $\actpol_2$ for $\spec_1$ and $\spec_2$, respectively. We will equip each action policy with a {\em bidding policy}, which it will use at runtime to bid for choosing the action at each time point. 
We formalize this below.

\stam{
We propose a decentralized approach to the problem. 
Our goal is to design two \emph{local} action policies $\actpol_1$ and $\actpol_2$ for $\spec_1$ and $\spec_2$, respectively.
Understandably, $\actpol_1$ and $\actpol_2$ may choose conflicting successors for extending a given finite path.
In order to decide whose choice to accept, an auction (aka bidding) is held, where the policies bid for the privilege of being \emph{scheduled} and choosing the next vertex. 
To add the bidding ability, we instrument each action policy with a \emph{bidding policy} and an initial \emph{budget} allocation, where the sum of two budgets is always normalized to $1$. 
At each point, the bidding policies simultaneously propose their bids not exceeding their current budgets. Whoever bids higher chooses the next vertex and ``pays'' the bid amount to the other policy, thereby updating budgets of both policies. We formalize the framework below. 
}


\begin{definition}[Bidding policies]
A \emph{bidding policy} is a function $\bidpol\colon \V\times [0,1]\to [0,1]$ with the constraint that $\bidpol(v,B)\leq B$ for every vertex $v$ and every amount of available budget $B\in [0,1]$.
\end{definition}

We equip a pair of action and bidding policies with a \emph{threshold budget}, which represents the greatest lower bound on the initial budget needed for the policies to guarantee their objective, and we call the resulting triple a {\em tender}. 

\stam{
Bidding policies are ineffective without sufficient budget.
We introduce \emph{tenders}, each of which encapsulates an action policy, a bidding policy, and the \emph{threshold budget} which represents the greatest lower bound on the initial available budget for the bidding and action policies to fulfill their objective.
The goal of our decentralized framework is to independently synthesize a pair of {\em tenders} which fulfill each objective separately.
}

\begin{definition}[Tenders]
	A \emph{tender} for a given graph $\T$ is a triple $\tup{\actpol,\bidpol,\B}$ of an action policy $\actpol$, a bidding policy $\bidpol$, and a \emph{threshold budget} $\B\in [0,1]$. The set of all tenders for $\G$ is denoted $\Tenders^\G$. 
A tender $\t$ \emph{satisfies} an objective $\spec$, denoted $\t\models \spec$, iff $\actpol\models \spec$ (i.e., when the tender is operating alone on the graph). 
\end{definition}

Next, we formalize the composition of two tenders at runtime, which produces an action policy that uses a register of memory to keep track of the available budgets. We introduce some notation. A {\em configuration} is a pair $\tup{v,B_1}$, where $v$ is a vertex and $B_1$ is the budget available to the first tender. We normalize the sum of budgets to $1$, hence implicitly, the budget available to the second tender is $B_2=1-B_1$.
Let $\C = \V\times [0,1]$ be the set of all configurations. For a given sequence of configurations $s = (v^0,B_1^0)(v^1,B_1^1)\ldots\in \C^\infty$, let $\proj{\V}{s}$ denote the path $v^0v^1\ldots$. 
A \emph{history} is a finite sequence of configurations $\tup{\vinit,B_1^0}\ldots\tup{v^k,B_1^k}\in \C^*$ with the constraint that $\proj{\V}{s}\in \finpaths(\G)$. Let $\H$ be the set of all histories.

\stam{
A single isolated tender is uninteresting.
We take two tenders and compose them via auction-based scheduling.
The composition produces an action policy that uses a register of memory to keep track of the available budgets. We introduce some notation. A {\em configuration} is a pair $\tup{v,B_1}$, where $v$ is a vertex and $B_1$ is the budget available to the first tender. We normalize the sum of budgets to $1$, hence implicitly, the budget available to the second tender is $B_2=1-B_1$.
Let $\C = \V\times [0,1]$ be the set of all configurations, and for a given sequence of configurations $s = (v^0,B_1^0)(v^1,B_1^1)\ldots\in \C^\infty$, let $\proj{\V}{s}$ denote the path $v^0v^1\ldots$. 
A \emph{history} is a finite sequence of configurations $\tup{\vinit,B_1^0}\ldots\tup{v^k,B_1^k}\in \C^*$ with the constraint that $\proj{\V}{s}\in \finpaths(\G)$. Let $\H$ be the set of all histories.
The composition of the tenders takes every history $hc \in \H$ ending at $c \in \C$, and extends it with the next configuration $c' = (v',B_1')$ obtained as follows.
The configuration $c$ is fed to the two bidding policies to obtain two bids, the highest bidder is determined, $\proj{\V}{hc}$ is fed to the highest bidder's action policy to obtain $v'$, and the bid amount is transferred from the highest bidder's budget to the lowest bidder's budget, updating the first tender's budget to $B_1'$. 
}

\begin{definition}[Composition of tenders]\label{def:composition}
	Let $\G$ be a graph, and $\t_1 = \tup{\actpol_1,\bidpol_1,\B_1}$ and $\t_2=\tup{\actpol_2,\bidpol_2,\B_2}$ be two tenders.
	The tenders $\t_1$ and $\t_2$ are \emph{compatible} iff $\B_1+\B_2 < 1$.
	If compatible, then their \emph{composition}, denoted $\compOT$, is a function $\comp{\t_1}{\t_2}\colon \H\to \C$ defined as follows. 
	Given a history $h=\tup{\vinit,B_1^0}\ldots\tup{v^k,B_1^k}\in \H$, let  $b_1\coloneqq\bidpol_1(v^k,B_1^k)$ and $b_2\coloneqq\bidpol_2(v^k,1-B_1^k)$. Then, 
	\begin{itemize}
		\item if $b_1\geq b_2$, then $\comp{\t_1}{\t_2}(h) = \left( \actpol_1(\rho v), B_1-b_1\right)$, and
		\item if $b_1 < b_2$, then $\comp{\t_1}{\t_2}(h)= \left(\actpol_2(\rho v), B_1+b_2\right)$.
	\end{itemize}
	Given an initial configuration $\tup{\vinit,B_1^0}$ with $B_1^0>\B_1$ and $B_2^0=1-B_1^0>\B_2$, the composition outputs an infinite sequence of configurations, denoted $\out(\compOT)$, where $\out(\compOT)\coloneqq \tup{\vinit,B_1^0}\tup{v^1,B_1^1}\ldots\in \C^\omega$ such that for every $k$, $\tup{v^k,B_1^k} = \compOT\left(\tup{\vinit,B_1^0}\ldots \tup{v^{k-1},B_1^{k-1}}\right)$.
	We will say $\compOT$ satisfies a given objective $\spec$, written $\compOT\models\spec$, iff $\proj{\V}{\out(\compOT)}\in \spec.$
\end{definition}
We will often use the index $i \in \set{1,2}$ to refer to either of the two tenders or their attributes, and will use $-i = 3-i$ for the ``other'' one, e.g., $\t_i$ and $\t_{-i}$. 
Notice the difference between $\B_i$ and $B_i^0$: $\B_i$ is the threshold budget at $\vinit$ which is a constant attribute of $\t_i$, whereas $B_i^0$ is the actual budget initially allocated to $\t_i$ whose value can be anything above $\B_i$.

\subsection{Classes of decentralized synthesis problem}
In this section, we describe three classes of decentralized synthesis problems that we study. 
Throughout this section, fix a graph $\G$ and a given pair of overlapping objectives $\spec_1$ and $\spec_2$.

\noindent{\bf Strong decentralized synthesis.}
Here, tenders make no assumptions on each other, thus the solutions provide the strongest (the most robust) guarantees. 
Formally, for each $i\in\set{1,2}$, the goal is to construct $\t_i$ such that for every compatible $\t_{-i}$, we have $\comp{\t_i}{\t_{-i}}\models\spec_i$.

\tikzset{%
glow/.style={%
preaction={#1, draw, line join=round, line width=0.5pt, opacity=0.04,
preaction={#1, draw, line join=round, line width=1.0pt, opacity=0.04,
preaction={#1, draw, line join=round, line width=1.5pt, opacity=0.04,
preaction={#1, draw, line join=round, line width=2.0pt, opacity=0.04,
preaction={#1, draw, line join=round, line width=2.5pt, opacity=0.04,
preaction={#1, draw, line join=round, line width=3.0pt, opacity=0.04,
preaction={#1, draw, line join=round, line width=3.5pt, opacity=0.04,
preaction={#1, draw, line join=round, line width=4.0pt, opacity=0.04,
preaction={#1, draw, line join=round, line width=4.5pt, opacity=0.04,
preaction={#1, draw, line join=round, line width=5.0pt, opacity=0.04,
preaction={#1, draw, line join=round, line width=5.5pt, opacity=0.04,
preaction={#1, draw, line join=round, line width=6.0pt, opacity=0.04,
}}}}}}}}}}}}}}
\tikzset{every state/.style={minimum size=0pt}}
\tikzset{%
target/.style={
preaction={#1, fill, opacity=0.5}
}}
\tikzset{cross/.style={
		preaction={#1, draw, cross out}}}

\begin{figure}
	\centering
	\begin{subfigure}[t]{0.3\textwidth}
		\centering
		\begin{tikzpicture}[node distance=0.7cm]
			\node[state,initial below]		(a)	at	(0,0)				{$a$};
			\node[state]		(b)	[left=of a]		{$b$};
			\node[state,target=blue]		(c)	[above=of b]		{$c$};
			\node[state,target=blue,target=red]		(d)	[below=of b]		{$d$};
			
			\node[state]		(e)	[right=of a]		{$e$};
			\node[state,target=red]		(f)	[above=of e]		{$f$};
			\node[state,target=blue]		(g)	[below=of e]		{$g$};
			
			\path[->,ultra thick]
				(a)		edge	[glow=blue]		(b)
				(b)		edge	[glow=red]		(d)
			;
						
			\path[->]
				(a)		edge		(e)
				(b)		edge			(c)
				(e)		edge	[glow=red]		(f)
						edge	[glow=blue]		(g);
		\end{tikzpicture}
		\caption{Strong}
		\label{fig:motivating example:robust}
	\end{subfigure}
	\begin{subfigure}[t]{0.3\textwidth}
		\centering
		\begin{tikzpicture}[node distance=0.7cm]
			\node[state,initial below]		(a)	at	(0,0)				{$a$};
			\node[state]		(b)	[left=of a]		{$b$};
			\node[state]		(c)	[above=of b]		{$c$};
			\node[state,target=blue,target=red]		(d)	[below=of b]		{$d$};
			
			\node[state]		(e)	[right=of a]		{$e$};
			\node[state,target=red]		(f)	[above=of e]		{$f$};
			\node[state,target=blue]		(g)	[below=of e]		{$g$};
			
			\path[->,ultra thick]
				(a)		edge	[glow=blue,glow=red]		(b)
				(b)		edge	[glow=red,glow=blue]		(d)
			;
						
			\path[->]
				(a)		edge		(e)
				(b)		edge			(c)
				(e)		edge	[glow=red]		(f)
						edge	[glow=blue]		(g);
		\end{tikzpicture}
		\caption{Assume-admissible}
		\label{fig:motivating example:admissible}
	\end{subfigure}
	\begin{subfigure}[t]{0.3\textwidth}
		\centering
		\begin{tikzpicture}[node distance=0.7cm]
			\node[state,initial below]		(a)	at	(0,0)				{$a$};
			\node[state]		(b)	[left=of a]		{$b$};
			\node[state,target=blue]		(c)	[above=of b]		{$c$};
			\node[state,target=blue,target=red]		(d)	[below=of b]		{$d$};
			
			\node[state]		(e)	[right=of a]		{$e$};
			\node[state,target=red]		(f)	[above=of e]		{$f$};
			\node[state,target=blue]		(g)	[below=of e]		{$g$};
			
			\path[->,ultra thick]
				(a)		edge	[glow=blue,glow=red]		(b)
				(b)		edge	[glow=red,glow=blue]		(d)
			;
						
			\path[->]
				(a)		edge		(e)
				(b)		edge			(c)
						edge		(f)
				(e)		edge	[glow=red]		(f)
						edge	[glow=blue]		(g);
		\end{tikzpicture}
		\caption{Assume-guarantee}
		\label{fig:motivating example:contract}
	\end{subfigure}
	\caption{Graphs with two reachability objectives given by targets: $T_\blue$, depicted in blue, $T_\red$ depicted in red, and $T_\blue \cap T_\red$ depicted in purple. The action policies of the red and blue tenders choose edges with, respectively,  red and blue shadows (shared edges are in purple). If no edges from a vertex have red or blue shadow, then the respective tender is indifferent about the choice made. Thick edges depict the paths taken by the compositions of tenders.}
	\label{fig:motivating example}
\end{figure}
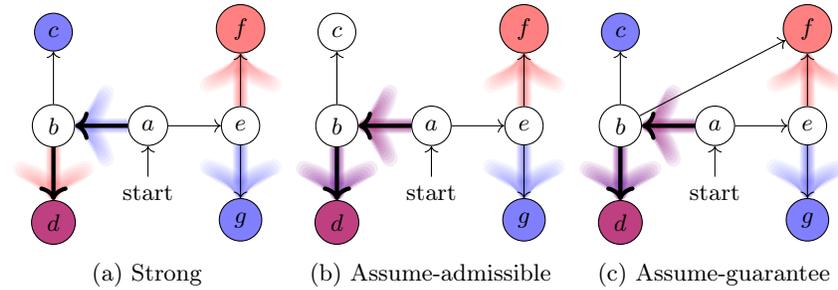
\begin{example}\label{ex:motivation:strong}
Consider the graph depicted in Fig.~\ref{fig:motivating example:robust} with a pair of {\em reachability} objectives having the targets $T_{\blue} = \set{c,d,g}$ and $T_{\red}=\set{d,f}$, respectively.  Their intersection $\set{d}$ is depicted in purple.
We present a pair of robust tenders $\t_{\blue}$ and $\t_{\red}$ with $\B_{\blue} = \sfrac{1}{4}$ and $\B_{\red}=\sfrac{1}{2}$, so that $\t_{\blue}$ and ${\t_{\red}}$ are compatible. We will show that $\t_{\blue}$ guarantees that no matter which compatible tender it is composed with, eventually $T_{\blue}$ is reached, and similarly $\t_{\red}$ ensures that $T_{\red}$ is reached. Therefore, $\comp{\t_{\blue}}{\t_{\red}}$ ensures that $d$ is reached.

We first describe $\t_{\blue}$. Consider an initial configuration $\zug{a, \sfrac{1}{4} + \epsilon}$, for any $\epsilon > 0$. Note that the other tender's budget is $\sfrac{3}{4}-\epsilon$. The first action of $\t_\blue$ is $\zug{b, \sfrac{1}{4}}$. There are two possibilities. First, $\t_{\blue}$ wins the bidding, then we reach the configuration $\zug{b, \epsilon}$, and since both successors of $b$ are in $T_{\blue}$, the objective is satisfied in the next step. Second, $\t_{\blue}$ loses the bidding, meaning that the other tender bids at least $\sfrac{1}{4}$, and in the worst case, we proceed to the configuration $\zug{e, \sfrac{1}{2}+\epsilon}$. Next, $\t_\blue$ chooses $\zug{g, \sfrac{1}{2}}$ and necessraily wins as $\t_\red$'s budget is only $\sfrac{1}{2} - \epsilon$, and we reach $g \in T_{\blue}$. We stress that $\t_\blue$ can be entirely oblivious about $\t_\red$, except for the implicit knowledge of $\t_\red$'s budget.

We describe $\t_{\red}$. Consider an initial configuration $\zug{a, \sfrac{1}{2} + \epsilon}$, for any $\epsilon > 0$. 
Initially, $\t_{\red}$ bids $0$, because it does not have a preference between going left or right. In the worst case, the budget stays $\sfrac{1}{2} + \epsilon$ in the next turn. 
Since both $b$ and $e$ have single successors in $T_{\red}$, thus $\t_\red$ must win the bidding. It does so by bidding $\sfrac{1}{2}$, which exceeds the available budget $\sfrac{1}{2}-\epsilon$ of $\t_\red$.\hfill$\triangle$
\end{example}

We now use the same problem as in Ex.~\ref{ex:motivation:strong}, and show that the conventional turn-based interaction may fail to fulfill both objectives.

\begin{example}\label{ex:turn-based fails}
	Consider again the graph depicted in Fig.~\ref{fig:motivating example:robust} with the targets $T_{\blue} = \set{c,d,g}$ and $T_{\red}=\set{d,f}$.
	Suppose $\actpol_\blue$ and $\actpol_\red$ are the two respective action policies, and we arbitrarily decide to make their interaction turn-based, where $\actpol_\red$ chooses actions in $a$ and $\actpol_\blue$ chooses actions in $b$ and $e$.
	It is clear that no matter which edge $\actpol_\red$ chooses from $a$, it cannot guarantee satisfaction of $T_\red$, because $\actpol_\blue$ can take the game to $c$ or $g$ depending on $\actpol_\red$'s choice.
	\hfill$\triangle$

\end{example}


\noindent{\bf Assume-admissible decentralized synthesis.}
While the guarantees of strong decentralized synthesis are appealing, it often fails as each tender makes the pessimistic assumption that the other tender can behave arbitrarily---even adversarially. 
We consider \emph{admissibility}~\cite{brenguier2015assume} as a stronger assumption based on rationality, ensuring compatible tenders to exist even when strong synthesis may fail. 
We illustrate the idea in the following example.

\begin{example}\label{ex:motivation:assume-admissible}
Consider the graph in Fig.~\ref{fig:motivating example:admissible}, with reachability objectives given by targets $T_{\blue}=\set{d,g}$ and $T_{\red}=\set{d,f}$.
We argue that strong decentralized synthesis is not possible. Indeed, using the same reasoning for $\t_\red$ in Ex.~\ref{ex:motivation:strong}, we have $\th_\blue(a) = \th_\red(a) = 0.5$. 
On the other hand, observe that when synthesizing $\t_\red$, since $c \notin T_\blue$, we know that a  ``rational''  $\t_\blue$---formally, {\em admissible} $\t_\blue$ (see Sec.~\ref{sec:assume-admissible})---will not proceed from $b$ to $c$, and we can omit the edge. In turn, the threshold in $a$ decreases to $\sfrac{1}{4}$ for both objectives. Since the sum of thresholds is now less than $1$, two compatible tenders can be obtained. \hfill$\triangle$
\end{example}

In general, we seek an {\em admissible-winning tender}, which ensures that its objective is satisfied when composed with any admissible tender. 
Admissible-winning tenders are modular because they can be reused provided that the set of admissible actions of the other tender remains unchanged. 
For example, even when vertex $g$ is added to the red target set, the blue tender can be used with no change. Somewhat surprisingly, we show that in graphs in which all vertices have out-degree at most $2$, assume-admissible decentralized synthesis is always possible, and a pair of admissible-winning tenders can be found in PTIME.

\noindent{\bf Assume-guarantee decentralized synthesis.}
Sometimes, even the admissibility assumption is too weak, and we need more direct synchronization of the tenders.
We consider assume-guarantee decentralized synthesis, where each tender needs to respect a pre-specified \emph{contract}, and as a result, their composition satisfies both objectives. We illustrate the idea below. 

\begin{example}\label{ex:motivation:assume-guarantee}
Consider the graph depicted in Fig.~\ref{fig:motivating example:contract}, with reachability objectives given by targets $T_{\blue}=\set{c,d,g}$ and $T_{\red}=\set{d,f}$. 
Here, the strong decentralized synthesis fails due to reasons similar to Ex.~\ref{ex:motivation:assume-admissible}.
The assume-admissible decentralized synthesis fails because from $e$, both objectives cannot be fulfilled, and from $b$, no matter which tender wins the bidding can use an admissible edge that violates the other objective (e.g., $(b,c)$ is admissible for $\t_\blue$ but violates $T_\red$).
We consider the {\em contract} $\tup{G_{\blue},G_{\red}}=\tup{\always\,\lnot c,\always\,\lnot f}$, which is satisfied when (a)~if $\actpol_{\blue}$ fulfills $G_{\blue}$, then $\actpol_{\red}$ fulfills $G_{\red}$, and 
(b)~if $\actpol_{\red}$ fulfills $G_{\red}$, then $\actpol_{\blue}$ fulfills $G_{\blue}$.
Now whichever tender wins the bidding at $b$ needs to fulfill its guarantee, because it cannot judge from the past interaction if the other tender violates its guarantee.
Therefore, from $b$, the next vertex will be $d$ under the contract, and using the same tenders from Ex.~\ref{ex:motivation:assume-admissible}, both objectives will be fulfilled.
\end{example}


\section{An Aside on Bidding Games on Graphs}
\label{sec:bidding games}

All our synthesis algorithms internally solve {\em bidding games}, which we briefly review here; see the survey~\cite{AH20} for more details. 
A (two-player) bidding game is played between \PO and \PT, and is a tuple $\zug{\G, \spec}$, where $\G = \zug{V, E}$ is the (finite, directed) graph and $\spec \subseteq V^\omega$ is the objective for \PO. 
The game is \emph{zero-sum}, meaning that the objective of \PT is $V^\omega \setminus \spec$, i.e., the violation of $\spec$.
This differs from auction-based scheduling where objectives overlap; otherwise, the interaction between \PO and \PT is the same as the one between tenders.
A {\em strategy} for a player is a pair $\zug{\actpol, \bidpol}$ where $\actpol$ is an action policy and $\bidpol$ is a bidding policy. As in the composition of tenders, two strategies and an initial configuration $\zug{v, B_1}$ give rise to an infinite sequence of configurations called a {\em play}. A strategy is {\em winning} if no matter which strategy the opponent follows, the play satisfies the player's objective. 
A central quantity in bidding games is the {\em threshold budget} in a vertex $v$, which is intuitively, a necessary and sufficient initial budget for \PO to guarantee winning. 

\begin{definition}[Threshold budgets]
Consider a bidding game $\zug{\G, \spec}$ with $\G=\tup{\V,\E}$. The {\em threshold} of \PO is given by $\th_\spec^\G: \V \rightarrow [0,1]$, where for every $v \in V$, we have $\th_\spec^\G(v) = \inf_B \set{\text{\PO has a winning strategy from } \zug{v, B}}$. 
\end{definition} 

The threshold of \PT is denoted as $\th_{\spec^c}^\G(v)$. The following theorem characterizes the structure of thresholds and states that the two players' thresholds are complementary.
We briefly describe the intuition in the below, following the theorem statement. 


\begin{restatable}[\cite{LLPU96}]{theorem}{thmRTreach}
\label{thm:RT-reach}
Consider a reachability bidding game $\tup{\G,\spec}$ where $\spec$ is $\reach_\G(T)$ where, without loss of generality, $T$ is a given set of sink vertices. For every vertex $v$, we have $\th_\spec^\G(v) = 1- \th_{\spec^c}^\G(v)$. Moreover, for every sink vertex $t$, we have $\th_\spec^\G(t) = 0$, if $t \in T$, and $\th_\spec^\G(t) = 1$ otherwise. For every vertex $v$, we have $\th_\spec^\G(v) = 0.5 \cdot (\th_\spec^\G(v^+) + \th_\spec^\G(v^-))$, where $v^-$ and $v^+$ are successors of $v$, such that for every other successor $u$, we have $\th_\spec^\G(v^-) \leq \th_\spec^\G(u) \leq \th_\spec^\G(v^+)$. 
Verifying if $\th_\spec^\G(v) > 0.5$ for a given vertex $v$ is in NP$\,\cap\,$coNP in general and is in PTIME for binary graphs. 
\end{restatable}

Fix a bidding game $\tup{\G,\spec}$ with the reachability objective $\spec = \reach_\G(\set{t})$ for the single target $t\in \V$. 
If the game starts at $t$, since the objective is already fulfilled---irrespective of the budget, hence $\th_\spec^\G(t)=0$.
For every other vertex $v$, it is known that the threshold can be represented as:
\begin{align}\label{eq:bidding games:reach threshold}
	\th_\spec^\G(v) = \frac{1}{2}\left( \th_\spec^\G(v^+) + \th_\spec^\G(v^-)\right),
\end{align} 
where $v^+$ and $v^-$ represent the successors of $v$ with the largest and the smallest thresholds, respectively.
Then if the initial configuration is $\tup{v,B} = \tup{v,\th_\spec^\G+\epsilon}$, for any given $\epsilon>0$, \PO can bid $\bidpol(v) = \sfrac{1}{2}(\th_\spec^\G(v^+) - \th_\spec^\G(v^-))$ and go to $\actpol(v) = v^-$ upon winning.
If he wins and moves the token to $v^-$, then his new budget will be $B - \bidpol(v) = \th_\spec^\G(v^-) + \epsilon$.
If he loses and \PT moves the token to a successor $v'$, then we know that \PT must have bid higher than $\bidpol(v)$ which \PO will receive, so that \PO's new budget will be at least $B + \bidpol(v) = \th_\spec^\G(v^+)+\epsilon$, larger than $\th_\spec^\G(v') + \epsilon$ by definition of $v^+$.
This way, the strategy $\tup{\actpol,\bidpol}$ can ensure that for every vertex $v'$ where the token may reach from $v$ in one step, the invariant $B>\th_\spec^\G(v')$ is maintained, so that \PO can win from $v'$, which can be repeated and eventually \PO will reach $t$. This strategy works on DAGs. For graphs with cycles, a little extra work on the bidding policy is needed to get out of cycles and move towards the target \cite{LLPSU99}.

For infinite-horizon objectives, like parity, it is known that eventually one of the BSCCs will be reached, and inside every BSCC every vertex can be reached by both players infinitely often with every arbitrary initial budget.
This implies that for every parity objective, the threshold of every vertex inside every BSCC in a game graph is either $0$ or $1$, and fulfilling a given parity objective is equivalent to \emph{reaching} a BSCC whose every vertex has threshold $0$.
We state this formally.

\begin{theorem}[\cite{AHC19}]
\label{thm:bidding-parity}
	Consider a bidding game $\tup{\G,\spec}$ with a parity objective $\spec$.
	Let $S$ be a BSCC of $\G$. Every vertex in $S$ has threshold either $0$ or $1$, and it is $1$ iff the highest parity index in $S$ is odd.
	Moreover, for a vertex $v$ not in a BSCC, we have $\th_\spec^\G(v)=\th_{\reach_\G(T)}^\G(v)$, where $T$ is the union of BSCCs whose vertices have threshold $0$.
\end{theorem}

\section{Strong Decentralized Synthesis}
\label{sec:strong}
We study the {\em strong decentralized synthesis} problem, where 
the goal is to synthesize two compatible {\em robust tenders}, i.e., tenders that guarantee the fulfillment of their objectives when composed with {\em any} compatible tender.

\begin{definition}[Robust tenders]\label{def:robust tenders}
	Let $\G$ be a graph and $\spec_i$ be an objective in $\G$.
	A tender $\t_i$ is \emph{robust} for $\spec_i$ if for every other compatible tender $\t_{-i}\in\Tenders^\G$, we have $\comp{\t_i}{\t_{-i}}\models \spec_i$.
\end{definition}

\begin{myproblem}[$\strongprob$]
Define $\strongprob$ as the problem whose input is a tuple $\tup{\G,\spec_1,\spec_2}$, where $\G$ is a graph and $\spec_1$ and $\spec_2$ are overlapping $\omega$-regular objectives in $\G$, and the goal is to decide whether there exists a pair of tenders $\t_1,\t_2  \in \Tenders^\G$ such that:
		\begin{enumerate}[(I)]
			\item $\t_1$ and $\t_2$ are compatible,
			\item $\t_1$ is robust for $\spec_1$, and 
			\item $\t_2$ is robust for $\spec_2$.
		\end{enumerate}
\end{myproblem}

Since each robust tender $\t_i$ guarantees that $\spec_i$ is satisfied when composed with {\em any} tender, the composition of two robust tenders satisfies both objectives:
\begin{proposition}[Sound composition of robust tenders]\label{prop:sound of composition of robust tenders}
Let $\t_1$ and $\t_2$ be two compatible robust tenders for $\tup{\G,\spec_1,\spec_2}$. Then $\comp{\t_1}{\t_2}\models \spec_1\cap\spec_2$.
\end{proposition}
We reduce the strong decentralized synthesis problem to the solution of two independent bidding games, both played on the graph $\G$, one with \PO's objective $\spec_1$ and the other one with \PO's objective $\spec_2$. When the sum of thresholds in $\vinit$ is less than $1$, we set the two tenders to be winning \PO strategies in the two games with the threshold budgets of the tenders being set as the respective thresholds in $v^0$. 
It follows from the construction that both tenders are robust, and hence their composition will fulfill both objectives (Prop.~\ref{prop:sound of composition of robust tenders}).


\begin{restatable}[Strong decentralized synthesis]{theorem}{StrongReach}\label{thm:strong:reach}
	Let $\G=\tup{\V,\vinit,\E}$ be a graph and $\spec_1$ and $\spec_2$ be a pair of overlapping $\omega$-regular objectives.
	A pair of robust tenders exists iff $\th_{\spec_1}^{\G}(\vinit) + \th_{\spec_2}^{\G}(\vinit) < 1$. Moreover, $\strongprob$ is in NP $\cap$ coNP in general and is in PTIME for binary graphs. 
\end{restatable}

\begin{proof}
	First, assume that $\th_{\spec_1}^{\G}(\vinit) + \th_{\spec_2}^{\G}(\vinit) < 1$. For $i \in \set{1,2}$, let $\zug{\actpol_i, \bidpol_i}$ denote a winning \PO strategy in the bidding game $\zug{\G, \spec_i}$ from every configuration $\zug{\vinit, B}$ with $B> \th_{\spec_i}^{\G}(\vinit)$. 
	We argue that the render $\t_1 = \zug{\actpol_1, \bidpol_1, \th_{\spec_1}^{\G}(\vinit)}$ is robust for $\spec_1$, and the proof for $\t_2$ is dual. Indeed, for any compatible tender $\t'_{2} = \zug{\actpol'_{2}, \bidpol'_{2}, \B'_{2}}$, the pair $\zug{\actpol'_{2}, \bidpol'_{2}}$ corresponds to a \PT strategy in the bidding game $\zug{\G, \spec_1}$. The resulting play coincides with $out(\comp{\t_1}{\t'_2})(\zug{\vinit, B})$ and satisfies $\spec_1$ since the strategy $\zug{\actpol_1, \bidpol_1}$ is winning. 
	
	Second, suppose that $\th_{\spec_1}^{\G}(\vinit) + \th_{\spec_2}^{\G}(\vinit) \geq 1$. For any allocation $\B_1 + \B_2 < 1$, there is an $i \in \set{1,2}$ such that $\B_i \leq \th_{\spec_i}^{\G}(\vinit)$. Assume WLog that $\B_1 \leq \th_{\spec_1}^{\G}(\vinit)$.
	Consider a winning \PT strategy $\zug{\actpol_2, \bidpol_2}$ in the bidding game $\zug{\G, \spec_1}$ from $\zug{\vinit, \B_1}$. The tender $\t'_{2} = \zug{\actpol_2, \bidpol_2, 1-\B_1}$ is compatible and $out(\comp{\t_1}{\t'_2}(\zug{\vinit, \B_1}))$ violates $\spec_1$. 
	
	Finally, in order to obtain the complexity bounds, we guess memoryless action policies in both games, which are known to exist~\cite{LLPSU99}, and verify that they are optimal. Based on the guess, we devise a linear program to compute the thresholds. Finally, we verify that the sum of thresholds in $\vinit$ is less than $1$. For binary graphs, there is no need to guess the action policy in order to find thresholds (Thm.~\ref{thm:RT-reach}).
	\qed
\end{proof}

We identify a setting where strong decentralized synthesis is always possible. The following theorem follows from the result that threshold budgets in strongly-connected B\"uchi games containing at least one accepting vertex are $0$.

\begin{theorem}[Strong decentralized synthesis on SCCs]
Consider a strongly-connected graph $\G$ and a pair of non-empty B\"uchi objectives in $\G$. Then, a pair of  robust tenders exists in $\G$.
\end{theorem}

We demonstrate the effectiveness of strong synthesis using path planning problems with two reachability objectives.
	Consider a fixed grid but four different instances of the problem, as shown in Fig.~\ref{fig:strong synthesis on a grid}.
	For the first three cases, we successfully obtain pairs of robust tenders whose compositions fulfill both objectives.
	Moreover, since the blue target remained the same in all cases, we needed to redesign only the red tender, saving us a significant amount of computation.

\begin{figure}
	\centering
	\includegraphics[width=0.24\textwidth]{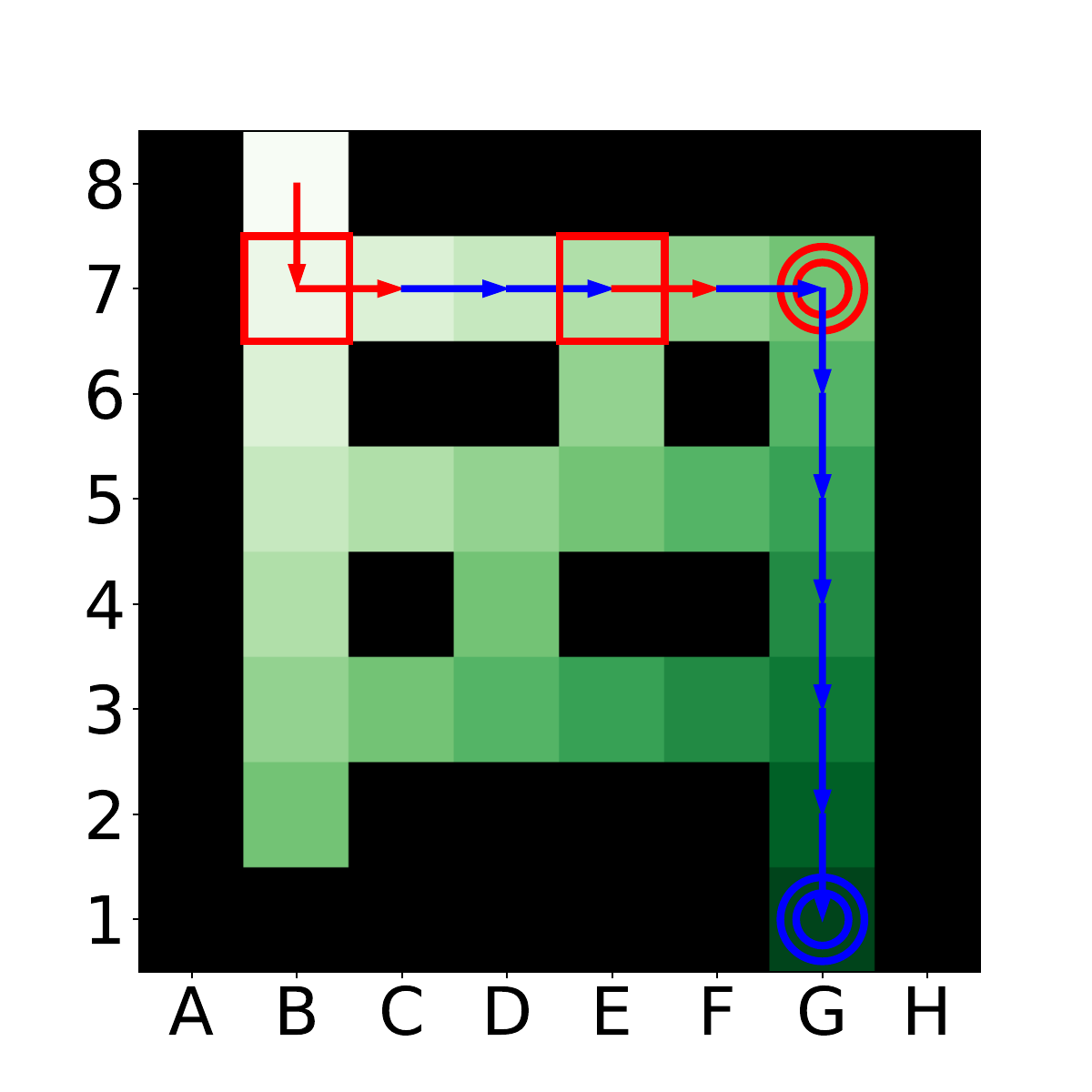}
	\includegraphics[width=0.24\textwidth]{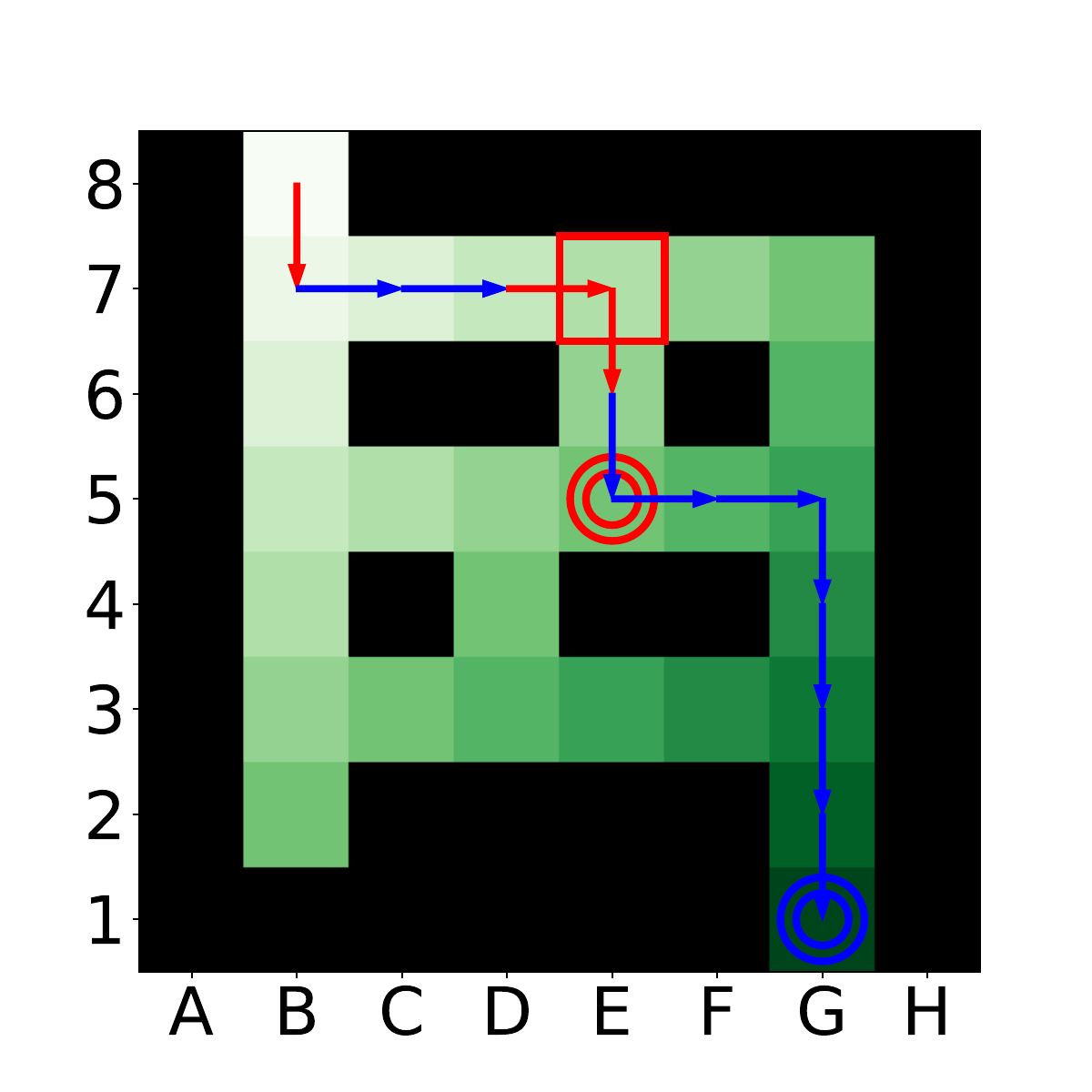}
	\includegraphics[width=0.24\textwidth]{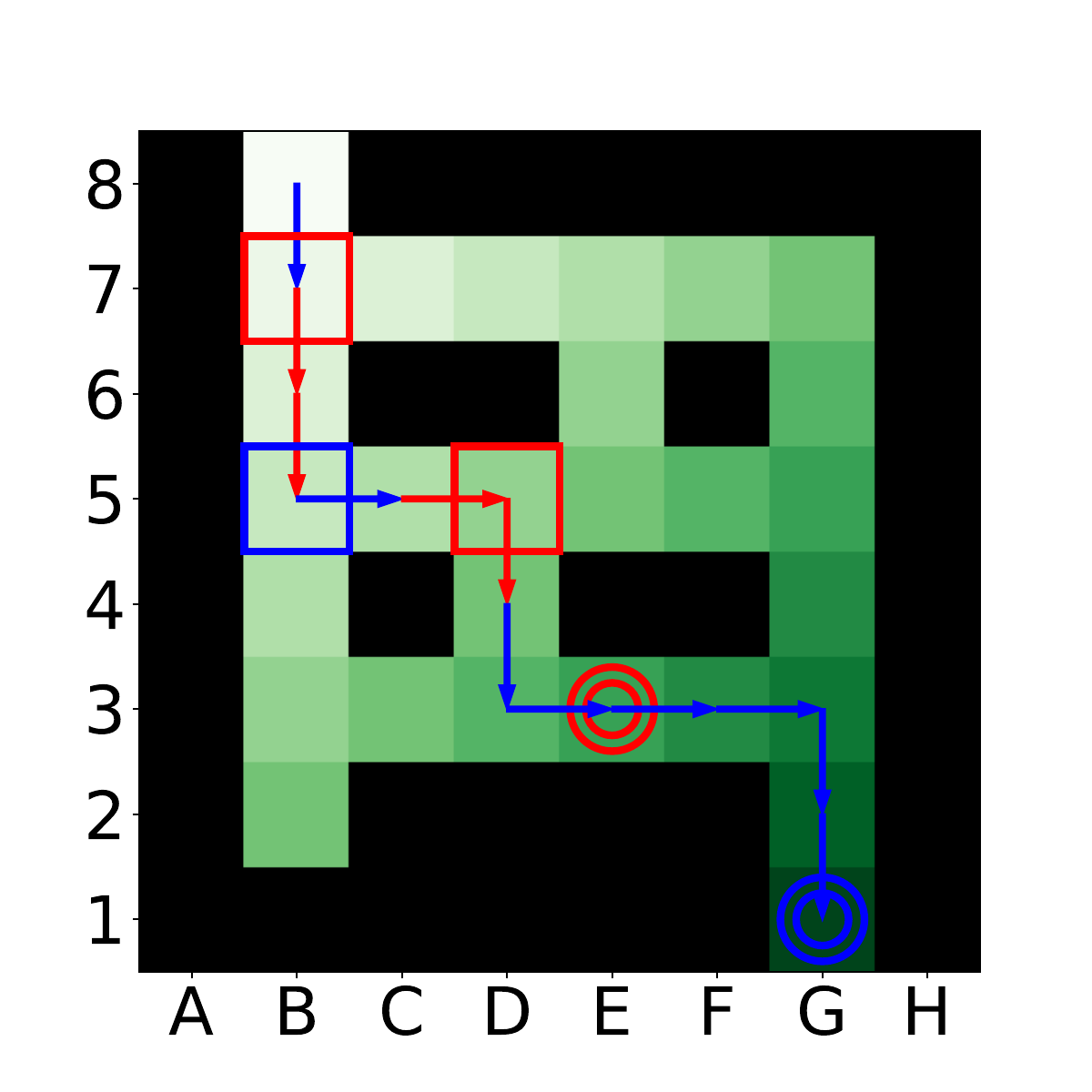}
	\includegraphics[width=0.24\textwidth]{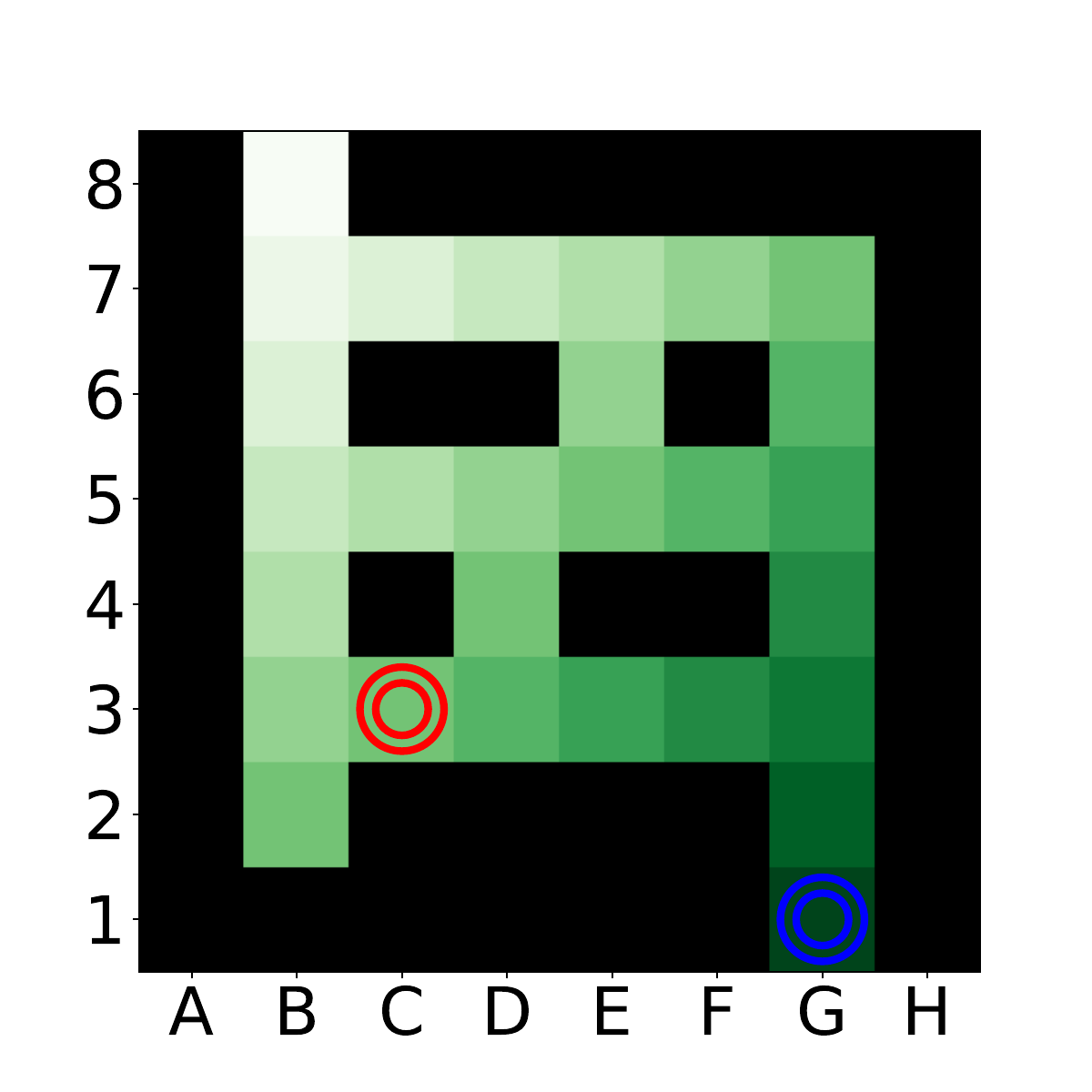}
	\caption{Robust tenders for path planning with two reachability objectives on a one-way grid, where the black cells are obstacles and \emph{the only permissible moves are from lighter to darker green cells}---and not the other way round.
	The cell B8 is the initial location.
	The cells with double circles of colors red (respectively, G7, E5, E3, C3) and blue (G1) are the targets to reach.
	The path shows the output of the composition of the two tenders, where the red and blue segments are actions which were chosen by the red and blue tenders, respectively.
	The cells with red and blue squares are locations where the respective tenders win the bidding; in the rest of the cells on the paths, the bidding ended in ties which were resolved randomly. 
	Strong synthesis was successful in the first three instances and failed in the last one.
	The pairs of thresholds of red and blue targets are, respectively (left to right):
	$(0.75,0.125)$,
	$(0.625,0.125)$,
	$(0.75, 0.125)$,
	$(0.875, 0.125)$.
	}
	\label{fig:strong synthesis on a grid}
\end{figure}


\section{Assume-Admissible Decentralized Synthesis}\label{sec:assume-admissible}

In assume-admissible decentralized synthesis, each tender assumes that the other tender is \emph{rational} and pursues its own objective.
We formalize rationality by adapting the well-known concepts of {\em dominance} and {\em admissibility} from game theory~\cite{BFH08,Berwanger07}. Intuitively, $\t_i$ dominates $\t'_i$ if $\t_i$ is always at least as good as $\t_i$ and sometimes strictly better than $\t_i$; therefore, there is no reason to use $\t_i'$.
 An admissible tender is one that is not dominated by any other tender.

\begin{definition}[Dominance, admissibility]\label{def:dominance}
	Let $\G$ be a graph and $\spec$ be an objective.
	We provide definitions for the first tender and the definitions for the second tender are dual. 
	Let $\B_1 < 1$. For two tenders $\t_1$ and $\t_1'$ that have the same budget allocation,  
	 $\t_1$ {\em dominates} $\t'_1$ when
	\begin{enumerate}[(a)]
	\item $\t_1$ performs as well as $\t'_1$ when composed with any compatible $\t_2$; formally, for every compatible tender  \(\t_2\),   
	 \(\comp{\t_1'}{\t_2} \models \spec\) implies \(\comp{\t_1}{\t_2} \models \spec\), and
	\item there is a compatible tender \(\t_2\) for which $\t_1$ performs better than $\t'_1$; formally, there exists a compatible $\t_2$ with \(\comp{\t_1}{\t_2} \models \spec\), and \(\comp{\t_1'}{\t_2} \not\models \spec\). 
	\end{enumerate}
	A tender $\t_1$ is called {\em $\spec$-admissible} in $\G$ iff it is not dominated by any other tender. We denote the set of $\spec$-admissible tenders in $\G$ by \(\Adm{\spec}{\G}\). 
\end{definition}

Next, we define admissible-winning tenders, which are tenders that fulfill their objectives when composed with {\em any} admissible tender. 

\begin{definition}[Admissible-winning tenders]\label{def:admissibletenders}
	Let $\G$ be a graph and $\spec_1,\spec_2$ be a pair of overlapping objectives in $\G$.
	A tender $\t_i$ is called \emph{$\spec_{-i}$-admissible-winning} for $\spec_i$ if and only if $\t_i\in\Adm{\spec_i}{\G}$, and for every other tender $\t_{-i}\in\Adm{\spec_{-i}}{\G}$ compatible with $\t_i$, we have $\comp{\t_i}{\t_{-i}}\models \spec_i$.
\end{definition}

When the objectives are clear from the context, we will omit them and will simply write a tender is ``admissible tender,'' ``admissible-winning tender,'' etc.

\begin{myproblem}[$\aaprob$]\label{prob:assume-admissible}
	Define $\aaprob$ as the problem whose input is a tuple $\tup{\G,\spec_1,\spec_2}$, 
	where $\G$ is a graph and
	$\spec_1$ and $\spec_2$ are overlapping $\omega$-regular objectives in $\G$,
	and the goal is to decide whether there exists a pair of tenders $\t_1 \in \Adm^\G(\spec_1)$ and $\t_2  \in \Adm^\G(\spec_2)$ such that:
		\begin{enumerate}[(I)]
			\item $\t_1$ and $\t_2$ are compatible,
			\item $\t_1$ is $\spec_2$-admissible-winning for $\spec_1$, and 
			\item $\t_2$ is $\spec_1$-admissible-winning for $\spec_2$.
		\end{enumerate}
\end{myproblem}

The following proposition follows from the requirement that \(\t_1\) and \(\t_2\) are admissible.
\begin{proposition}[Sound composition of admissible-winning tenders]\label{prop:aa-sound}
Let $\t_1$ and $\t_2$ be tenders that fulfill the requirements stated in Prob.~\ref{prob:assume-admissible}. Then, $\comp{\t_1}{\t_2}\models \spec_1\cap\spec_2$.
\end{proposition}

\begin{remark}
Note that the synthesis procedure for each $\spec_{-i}$-admissible-winning tender $\t_i$ for $\spec_i$ requires the knowledge of $\Adm{\spec_{-i}}{\G}$. 
Assume-admissible decentralized synthesis is modular in the following sense. 
First, the specific implementation of the tender $\t_{-i}$ with which each $\t_i$ is composed is not known during synthesis. All that is known is the objective $\spec_{-i}$ for which $\t_{-i}$ is synthesized. Second, each $\t_i$ can remain unchanged even when $\spec_{-i}$ changes to $\spec_{-i}'$, as long as $\Adm{\spec_{-i}'}{\G} \subseteq \Adm{\spec_{-i}}{\G}$.
\end{remark}

\subsection{Reachability objectives}

Throughout this section we focus on overlapping reachability objectives $\spec_1 = \reach_\G(T_1)$ and $\spec_2 = \reach_\G(T_2)$ with $T_1,T_2\subseteq \V$ being sets of sink target vertices. 
This is without loss of generality, as every graph with non-sink target vertices can be converted into a graph with sink target vertices by adding memory. 
We briefly describe the intuition in the following. 
A pair of reachability objectives with non-sink targets can be turned into a pair of reachability objectives $\spec_1',\spec_2'$ with sink targets $T_1',T_2'$ on a larger graph $\G'$.
The graph $\G'$ contains three copies of $\G$, call them $\G_\bot$, $\G_1$, and $\G_2$, which memorize which objective has been fulfilled so far, if any.
Initially the token is placed on $\vinit$ of $\G_\bot$, and as soon as either $T_1$ or $T_2 $ is visited, the token is moved to the respective vertex in $G_1$ or $G_2$, respectively.
In each $\G_i$, every sink vertex belongs to $T_i'$.
Furthermore, if there exists a $t\in T_i$, then the copy of $t$ in $G_{-i}'$ belongs to $T_i'$.
It can be shown that a path in $\G$ fulfills $\spec_1\cap\spec_2$ iff the corresponding path in $\G'$ fulfills $\spec_1'\cap\spec_2'$.

We reduce the decentralized assume-admissible synthesis problem to solving a pair of zero-sum bidding games on a sub-graph of $\G$. 
 Intuitively, an edge $e = \zug{u,v}$ is {\em dominated} for the $i$-th tender, for $i \in \set{1,2}$, if it is possible to achieve the objective $\spec_i$ from $u$ but not from $v$. Clearly, a tender that chooses $e$ is dominated and is thus not admissible (see the proof of the lemma in the full version \cite{AMS23}).  Recall that $\thresh{v}{\spec_i}{\G}$ denotes the threshold in the zero-sum bidding game played on $\G$ with the \PO objective $\spec_i$, and that $\th_{\spec_i}^\G(v)=1$ means there is no path from $v$ to $T_i$.

\begin{restatable}[A necessary condition for admissibility]{lemma}{necessarycondn}\label{lem:necessary conditions for admissibility}
	For every vertex $u$ having at least two successors $v,w$ with $\th_{\spec_i}^\G(v)< 1$ and $\th_{\spec_i}^\G(w)=1$, if a \PLi tender $\tup{\actpol_i,\bidpol_i, \B_i} $ is in $ \Adm{\spec_i}{\G}$, then $\actpol_i(u)\neq w$, for both $i\in\set{1,2}$.
\end{restatable}

\begin{proof}
	We argue that choosing $w$ from $u$ is dominated by the action of choosing $v$ from $u$, no matter what the budget at $u$ is.
	Firstly, Cond.~(a) of Def.~\ref{def:dominance} trivially holds.
	Secondly, consider the other tender $\t_{-i}$ which bids zero at $u$, and later cooperates with $\t_i$ to satisfy $\spec_i$.
	Clearly, the $\t_i$'s action policy that selects $v$ at $u$ will be able to satisfy  $\spec_i$, but the one that selects $w$ will not.
	\qed
\end{proof}

We obtain the reduced graph by omitting edges that are dominated for both players. 
For example, in Fig.~\ref{fig:motivating example:admissible}, the edge $\zug{b,c}$ is dominated for both players (see Ex.~\ref{ex:motivation:assume-admissible}) and in Fig.~\ref{fig:motivating example:contract}, no edge is dominated for both players (see Ex.~\ref{ex:motivation:assume-guarantee}).

\begin{definition}[Largest admissible sub-graphs for reachability]
The largest admissible sub-graph of $\G$ with respect to two reachability objectives $\spec_1$ and $\spec_2$  is  $\las{\G}{\spec_1}{\spec_2}=\tup{\V',\E'}$ with 
$\V' = \V\setminus \set{ v\in\V \mid \th_{\spec_1}^\G(v)=1 \land \th_{\spec_2}^\G(v)=1}$ and $\E' = (\V'\times \V') \cap \E$.
We omit \(\spec_1, \spec_2\) from \(\las{\G}{\spec_1}{\spec_2}\) when it is clear from the context. 
\end{definition}

For a vertex $v$ in $\G$ and $i \in \set{1,2}$, recall that $\thresh{v}{\G}{\spec_i}$ denotes the threshold in $\G$ for objective $\spec_i$, and $\thresh{v}{\las{\G}}{\spec_i}$ denotes the threshold in the reduced graph. Observe that a winning strategy in $\G$ will never cross a dominated edge. Removing dominated edges restricts the opponent, thus $\thresh{v}{\G}{\spec_i} \geq \thresh{v}{\las{\G}}{\spec_i}$. The next lemma shows that, surprisingly, the decrease in sum of thresholds is guaranteed to be significant. The proof which holds for non-binary graphs, intuitively follows from observing that in $\las{\G}$, necessarily a sink that is a target for one of the players is reached, and since there is an overlap in at least one sink, the sum of thresholds is at most $1$.

\begin{restatable}[On the sum of thresholds in $\las{\G}$]{lemma}{lemAaReach}\label{lem:aa:reach:sum of thresholds leq 1}
	For every vertex $v$, we have $\th_{\spec_1}^{\las{\G}}(v) + \th_{\spec_2}^{\las{\G}}(v) \leq 1$. Moreover, if $\G$ is binary then $\th_{\spec_1}^{\las{\G}}(v) + \th_{\spec_2}^{\las{\G}}(v) < 1$.
\end{restatable}

\begin{proof}
	Let $\G,\las{\G},\spec_1=\reach_{\las{\G}}(T_1),\spec_2=\reach_{\las{\G}}(T_2)$ are as in the statement of the lemma.
	By the definition of thresholds, we know that $\th_{\spec_1}^{\las{\G}}(v) = 1-\th_{\spec_1^c}^{\las{\G}}(v)$.
	Since every target vertex in $T_1$ and $T_2$ is a sink vertex, and since we removed every vertex $v$ with $\th_{\spec_1}^\G(v)=\th_{\spec_2}^\G(v)=1$ in $\las{\G}$, hence we know that $\las{\G}$ does not have any sink vertex that is outside of $T_1\cup T_2$, as all such vertices have threshold $1$ for both objectives and got removed.
	Moreover, because $T_1 \cap T_2\neq\emptyset$, hence $T2 \subseteq T_1^c$, implying $\spec_2 \subseteq \spec_1^c$.
	It can be shown, by inductively applying the threshold operator from Eq.~\eqref{eq:bidding games:reach threshold} starting at the sink vertices and reaching up to the root vertices, that $\th_{\spec_2}^{\las{\G}}(v) \leq \th_{\spec_1^c}^{\las{\G}}(v)$.
	Therefore, we have $\th_{\spec_1}^{\las{\G}}(v) \leq 1- \th_{\spec_2}^{\las{\G}}(v)$, implying the claim.

	We proceed to show that when $\G$ is a binary graph, we have $\th_{\spec_1}^{\las{\G}{\spec_1}{\spec_2}}(v) + \th_{\spec_2}^{\las{\G}{\spec_1}{\spec_2}}(v) < 1$.
	The proof is via induction over the length of the shortest path from $v$ to any vertex in $T_1\cap T_2$, which is unique and bounded (by the number of vertices) for every $v$ and henceforth called the \emph{distance} between $v$ and $T_1\cap T_2$.
	The base case is when the distance is $0$, i.e., for every vertex $v$ in $T_1\cap T_2$, for which we immediately obtain $\th_{\spec_1}^{\las{\G}}(v) + \th_{\spec_2}^{\las{\G}}(v) = 0 + 0 < 1$.
	Now suppose for every vertex $v$ with distance $d$, the claim holds.
	For every vertex $v$ with distance $d+1$, we know that $v$ has a successor $v'$ whose distance is $d$, so that $\th_{\spec_1}^{\las{\G}}(v') + \th_{\spec_2}^{\las{\G}}(v') < 1$.
	Let $v''$ be the other successor of $v$ with any arbitrary distance.
	Using the characterization of thresholds from \eqref{eq:bidding games:reach threshold}, and observing that for two successors, the distinction between $v^-$ and $v^+$ is immaterial:
	\begin{align*}
		\th_{\spec_1}^{\las{\G}}(v) + \th_{\spec_2}^{\las{\G}}(v)  &= 
		\frac{1}{2}\left( \th_{\spec_1}^{\las{\G}}(v') + \th_{\spec_1}^{\las{\G}}(v'') \right) + 
		\frac{1}{2}\left( \th_{\spec_2}^{\las{\G}}(v') + \th_{\spec_2}^{\las{\G}}(v'') \right)\\
		&= \frac{1}{2}\underbrace{\left(\th_{\spec_1}^{\las{\G}}(v') + \th_{\spec_2}^{\las{\G}}(v')\right)}_{< 1 \text{ (induction hypothesis)}} + 
		\frac{1}{2}\underbrace{\left(\th_{\spec_1}^{\las{\G}}(v'') + \th_{\spec_2}^{\las{\G}}(v'')\right)}_{\leq 1 \text{ (Lem.~\ref{lem:aa:reach:sum of thresholds leq 1})} }\\
		&< \frac{1}{2}\cdot 1 + \frac{1}{2}\cdot 1\\
		&= 1. 
	\end{align*} \qed
\end{proof}

Our synthesis procedure proceeds as in strong decentralized synthesis: Find and output a pair of robust tenders in $\las{\G}$, which are guaranteed to exist when $\G$ is binary. In order to maintain soundness (see Prop.~\ref{prop:aa-sound}), it is key to show that a robust tender $\t_i$ in $\las{\G}$ is admissible in $\G$. The proof of the following lemma is intricate. We show that even when one can find $\t'_i$ and $\t_{-i}$ such that $\comp{\t_i}{\t_{-i}} \not \models \spec_i$ but $\comp{\t_i'}{\t_{-i}}  \models \spec_i$, it is possible to construct $\t_{-i}'$ for which $\comp{\t_i}{\t_{-i}'} \models \spec_i$ but $\comp{\t_i'}{\t_{-i}'} \not \models \spec_i$, thus $\t'_i$ does not dominate $\t_i$. 
Furthermore, such a tender wins against a set of tenders which \emph{over-approximates} admissible tender for \(\spec_{-i}\). 

\begin{restatable}[Algorithm for computing admissible-winning tenders]{lemma}{robustToAdmissible}
	\label{lemm: robust-to-admissible}
For $i \in \set{1,2}$, a robust tender for $\spec_i$ in $\las{\G}$ is \(\spec_{-i}\)-admissible-winning for $\spec_i$ in $\G$. 
\end{restatable}

\begin{proof}
	We prove it for \(i = 1\), the other case is dual.
	Suppose \(\t_1 = \tup{\actpol, \bidpol, \B}\) is a tender that follows a \PO winning strategy in the game $\zug{\las{\G}, \spec_1}$. 
	Suppose towards contradiction that \(\t_1\) is dominated by a tender \(\t_1' = \tup{\actpol', \bidpol', \B}\). 
	That is, there is a tender \(\t_2 = \tup{\actpol_2, \bidpol_2, \B_2}\) that is compatible with both such that \(\comp{\t_1}{\t_2} \not\sat \spec_1\) but \(\comp{\t_1'}{\t_2} \sat \spec_1\). We will reach a contradiction by constructing another compatible tender \(\t_2' = \tup{\actpol_2', \bidpol_2', \B_2}\) that achieves the opposite, namely \(\comp{\t_1}{\t_2'} \sat \spec_1\) but \(\comp{\t_1}{\t_2'} \not\sat \spec_1\).
	
	We construct $\t_2'$ as follows. 
	Let \(h\) be the longest common prefix of \(\comp{\t_1}{\t_2}\) and \(\comp{\t_1'}{\t_2}\), and the last configuration of \(h\) is \(\zug{v, B_v}\). 
	Denote the actions of $\t_1$ and $\t_1'$ following $h$ respectively by $\zug{b_1, u_1}$ and $\zug{b_1', u_1'}$. Note that these differ since $h$ is the longest common prefix. We construct $\t'_2$ to follow the same actions and bids as $\t_2$ so that the prefix $h$ is generated.
	Let $\zug{u_1, B_1} = \comp{\t_1}{\t_2}(h)$ and $\zug{u_1', B_1'} = \comp{\t'_1}{\t_2}(h)$. 
	Since $h$ is maximal, $\zug{u, B_1} \neq \zug{u', B_1'}$. We show that both budgets are below the adversarial threshold, i.e., $\thresh{u}{\G}{\spec_1} > B_1$ and $\thresh{u'}{\G}{\spec_1} > B_1'$. This means that we can set $\t'_2$ to follow a winning \PT strategy in the game $\zug{\G, \spec_1}$ from $\zug{u_1', B_1'}$ to violate $\spec_1$. On the other hand, we set $\t'_2$ to follow some \PT strategy in $\las{\G}$ from $\zug{u_1,B_1}$, and since $\t_1$ is winning in $\las{\G}$, the objective $\spec_1$ is satisfied. 
	Let $\zug{b_2, u_2}$ denote the action of $\t_2$ following $h$. 
	Note that it cannot be case that $b_2 > b_1, b'_1$ since otherwise, $\comp{\t_1}{\t_2}(h) = \comp{\t'_1}{\t_2}(h)$. 
	We show the claim separately for each of the four remaining orders of $b_1, b'_1$, and $b_2$.

	\begin{enumerate}[(i)]
		\item  \textbf{\(b_1 < b_2 < b_1'\)}.
		\(\t_2'\) bids \(b_2\) as well, and choose whatever \(\t_2\) does. 
		Therefore, \(\t_2'\), after losing the bid at \(v\), can simulate an adversarial strategy for \(\spec_1\), and does so (just like \(\t_2\) does against \(\t_1\)). 
		This makes \(\comp{\t_1'}{\t_2'} \not\models \spec_1\). 
		On the other hand, upon winning the bid at \(v\), \(\t_2'\), and remain within the longest admissible sub-graph \(\las{G}\) of \(\G\). 
		In the latter case, \(\t_1\) is a winning strategy, therefore \(\comp{\t_1}{\t_2'} \models \spec_1\). 
		
		\item \textbf{\(b_1' < b_2 < b_1\)}. 
		Again, \(\t_2'\) bids \(b_2\) as well, but upon winning the bid at \(v\), \(\actpol_2'\) chooses \(v^+\).
		Because \(B_v < \thresh{v}{\spec_1}{\G}\), and \(b_1\) is an optimal bid, when \(\t_1'\) loses the bid again \(\t_2'\), its budget becomes \(B_v + b_2 < \frac{\thresh{v^+}{\spec_1}{\G} + \thresh{v^-}{\spec_1}{\G}}{2} +  \frac{\thresh{v^+}{\spec_1}{\G} - \thresh{v^-}{\spec_1}{\G}}{2} = \thresh{v^+}{\spec_1}{\G}\). 
		Therefore, from \(v^+\), \(\t_2'\) can simulate an adversarial strategy for \(\spec_1\), and does so, making \(\comp{\t_1'}{\t_2'} \not\models \spec_1\). 
		On the other hand, after losing the bidding at \(v\) against \(\t_1\), from \(u\) onwards, \(\t_2'\) only chooses edges within \(\las{G}\). 
		Again, as \(\t_1\) wins against any tender within \(\las{G}\), we have \(\comp{\t_1}{\t_2'} \models \spec_1\). 
		
		\item \(b_2 < b_1 < b_1'\). 
		Here, again \(\t_2'\) simulates \(\t_2\) up to \(v\), and loses the current bid against both \(\t_1\) and \(\t_1'\). 
		But \(t_1'\) remains in a better position (budget wise) than in \(\t_1\). 
		Against \(\t_1'\), \(\t_2\) does what \(\t_2\)'s choices against \(\t_1\). 
		This is well defined, because \(\t_2'\) in a better position (budget-wise) \(\t_2\) to do that, and as a result \(\comp{\t_1'}{\t_2'} \not\models \spec_1\). 
		On the other hand, against \(\t_1\), \(\t_2'\) only chooses edges within \(\las{G}\) from \(u\), making \(\comp{\t_1}{\t_2'} \models \spec_1\). 
		
		\item \(b_2 < b_1' < b_1\). 
		If \(\t_2'\) has a budget such that it can bid some \(b_2'\) satisfying \(b_1' < b_2' < b_1\), \(\bidpol_2'\) does that, thereafter \(\t_2'\) follows as it does in case (ii).
		Otherwise, \(\t_2'\)'s (and \(\t_2\)) budget at \(v, B_2 = 1 - B_v < b_1' < b_1\). 
		Now, \(\t_1\) is by construction an optimal tender for \(\spec_1\). 
		The optimal bid \(b_1\) is strictly more than \(\t_1\)'s budget means \(\thresh{v}{\spec_1}{\G'} = 1\), as then only it becomes urgent to win the current  bid for \(\t_1\). 
		Therefore, we have \(\thresh{v}{\spec_1}{\G'} = (1 + \thresh{v^-}{\spec_1}{\G'})/2 \implies \thresh{v^-}{\spec_1}{\G'} = \thresh{v}{\spec_1}{\G'} - (1 - \thresh{v}{\spec_1}{\G'})\). 
		We already assumed, \(B_v < \thresh{v}{\spec_1}{\G}\). 
		Therefore, 
		\begin{align*}
			B_v - b_1' &< \thresh{v}{\spec_1}{\G} - (1 - B_v) \\
			&< \thresh{v}{\spec_1}{\G} - (1 - \thresh{v}{\spec_1}{\G})\\
			&= \thresh{v^-}{\spec_1}{\G}
		\end{align*}
		Therefore, even when \(\t_2'\) does not have enough budget to bid between \(b_1\) and \(b_1'\), upon losing the bid at \(v\), \(\t_2'\) can choose to play an adversarial strategy from \(v^-\) against \(\t_1'\) has budget \(B_1' = B_v - b_1'\), and choose to play edges only within \(\G'\) against \(t_1\) when it has a budget \(B_1 = B_v - b_1\). 
		Thus, \(\comp{\t_1}{\t_2'} \models \spec_1\), but \(\comp{\t_1'}{\t_2'} \not\models \spec_1\).
	\end{enumerate}
	
	Now, in Lem.~\ref{lem:necessary conditions for admissibility}, we showed that the no admissible tender (for either specifications) takes the edges that we remove from \(\G\) in order to construct \(\las{G}\), i.e, any \(\spec_i\)-admissible tenders can only constitute of edges remaining in \(\G'\) for both \(i\). 
	As a robust tender \(\t_i\) for \(\spec_i\) from \(\las{G}\), when composed with any tenders constituted only by edges present in \(\las{G}\), satisfies \(\spec_i\), we can say that it does so when it is composed with an \(\spec_2\)-admissible tenders (which is a subset of tenders present in \(\las{\G'}\)). \qed
\end{proof}

The following theorem is obtained by combining Lemmas~\ref{lem:aa:reach:sum of thresholds leq 1} and~\ref{lemm: robust-to-admissible}.

\begin{restatable}[Assume-admissible decentralized synthesis for reachability]{theorem}{assumeadmissiblethm}\label{thm:aa:reach}
	The problem  $\aaprob$ is a tautology for  binary graphs: for every binary graph and two overlapping reachability objectives, there exists a pair of compatible admissible-winning tenders. Moreover, the tenders can be computed in PTIME. 
\end{restatable}

\begin{remark}\label{rem:AA:reach:non-binary}
For general (i.e., non-binary) graphs, $\aaprob$ is not a tautology anymore; a counter-example is given in Ex.~\ref{ex:motivation:assume-guarantee}.
However, the same PTIME algorithm for computing tenders can still be used to obtain a sound solution; the completeness question is left open for future work.
\end{remark}


\subsection{B\"uchi objectives}

In this section, we consider binary graphs with a pair of overlapping B\"uchi objectives. 
We first demonstrate that, unlike reachability, it is not guaranteed that an assume-admissible decentralized solution exists. 

\begin{wrapfigure}{r}{6.2cm}
		\centering
		\vspace{-1cm}
		\begin{tikzpicture}
			\node[state,target=blue]		(a)		at	(0,0)	{$a$};
			\node[state,target=red,initial above]		(b)		[right=of a]		{$b$};
			\node[state,target=blue]		(c)		[right=of b]		{$c$};
			\node[state,target=red]		(d)		[right=of c]		{$d$};
			
			\path[->]
				(a)		edge	[loop left]		()
				(b)		edge					(a)
						edge[bend left]		(c)
				(c)		edge[bend left]		(b)
						edge					(d)
				(d)		edge[loop right]		();
		\end{tikzpicture}
		\caption{A graph with no assume-admissible decentralized solution.}
		\label{fig:aa-buchi:four state}
		\vspace{-0.6cm}
	\end{wrapfigure}
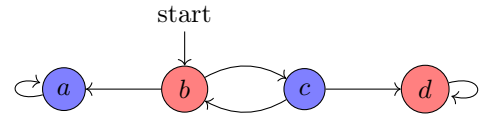
\begin{example}\label{ex:AA-buchi}
	\tikzset{every state/.style={minimum size=0pt}}
\tikzset{%
target/.style={
preaction={#1, fill, opacity=0.5}
}}
	Consider the graph depicted in Fig.~\ref{fig:aa-buchi:four state} with the B\"uchi objectives given by the accepting vertices $S_\red = \set{b,d}$ and $S_\blue = \set{a,c}$. Note that the objectives are overlapping since the path $(bc)^\omega$ satisfies both. We argue that no pair of compatible admissible-winning tenders exist. Note that a robust (hence dominant) red tender forces reaching $d$, thus forcing $\spec_\red$ to be satisfied. Dually, a robust blue tender forces $\spec_\blue$ in $a$. It can be shown that $\thresh{b}{\spec_\red}{\G} = \sfrac{2}{3}$ and $\thresh{b}{\spec_\blue}{\G} = \sfrac{1}{3}$. Thus, for any $\B_\red$ and $\B_\blue$ with $\B_\red+\B_\blue<1$, there is a robust tender that violates the other tender's objective. \hfill$\triangle$
\end{example}

We generalize the concept of largest admissible subgraphs to B\"uchi objectives. 
It is not hard to show that proceeding into a BSCC with an accepting state is admissible. Indeed, Thm.~\ref{thm:bidding-parity} shows that there is a robust (hence admissible) tender in such BSCCs. On the other hand, proceeding to a BSCC with no accepting vertex is clearly not admissible. The largest admissible subgraph is obtained by repeatedly removing BSCC that are not admissible for both tenders. Formally, 
for a given action policy $\actpol$ and a given vertex $v$ of $\G$, we will write $\actpol\not\models_v \spec_1\cup \spec_2$ to indicate that the action policy cannot fulfill $\spec_1\cup\spec_2$ from the initial vertex $v$.

\begin{definition}[Largest admissible sub-graphs for B\"uchi]
	The largest admissible sub-graph $\lasb{\G}$ of $\G$ for the B\"uchi objectives $\spec_1,\spec_2$ is the graph $\tup{\V',\E'}$ with $\V' = \V\setminus \set{ v\in\V \mid \forall \text{ action policy } \actpol\;.\;\actpol\not\models_{v}\spec_1\cup \spec_2}$, and $\E' = (\V'\times \V') \cap \E$.
\end{definition}

We describe a reduction to reachability games. 
For $i \in \set{1,2}$, let $T_i$ denote the union of BSCCs of $\lasb{\G}$ in which there is at least one B\"uchi accepting vertex for $\spec_i$. 
We call $\tup{T_1,T_2}$ the \emph{reachability core} of $\tup{\spec_1,\spec_2}$ in $\lasb{\G}$. Let $\spec_1' = \reach_{\lasb{\G}}(T_1)$ and $\spec_2' = \reach_{\lasb{\G}}(T_2)$. 
We proceed as in strong decentralized synthesis: we find $\thresh{\vinit}{\spec_1'}{\lasb{\G}}$ and $\thresh{\vinit}{\spec_2'}{\lasb{\G}}$ and return a pair of robust tenders if their sum is strictly less than $1$. Note that unlike reachability objectives, in B\"uchi objectives the sum might be $1$ as in Ex.~\ref{ex:AA-buchi}. 
Moreover, as Ex.~\ref{ex:AA-buchi} illustrates, when the sum is $1$, no pair of admissible-winning tenders exist. By adapting results from the previous section, we obtain the following.

\begin{restatable}[Assume-admissible decentralized synthesis for B\"uchi]{theorem}{thmBuchi}
	\label{thm:buchi}
	Let $\G$ be a binary graph and $\spec_1,\spec_2$ be a pair of overlapping B\"uchi objectives.
	Let $\tup{T_1,T_2}$ be the reachability core of $\tup{\spec_1,\spec_2}$ in the largest admissible sub-graph of $\G$ (for $\spec_1,\spec_2$).
	A pair of admissible-winning tenders exists iff $T_1\cap T_2\neq \emptyset$.
	Moreover, $\aaprob$ for B\"uchi objectives is in PTIME.
\end{restatable}

\begin{proof}
	Let us suppose \(\spec_i = \buchi(\buchiset_i)\) for \(i = \{1, 2\}\). 
	We consider \(\G'\) that is obtained from \(\lasb{G}\) after replacing every BSCC \(D\) with sink \(d\), as explained before.  
	From Thm.~\ref{thm:aa:reach}, \(\zug{G', T_1, T_2}\) is in \(\aaprob\). 
	Therefore, we get a pair of tenders \(\t_1' = \zug{\actpol_1', \bidpol_1', \B_1}\) and \(\t_2' = \zug{\actpol_2', \bidpol_2', \B_2}\) as certificates, and by design they are compatible. 
	We now construct tenders \(\t_1 = \zug{\actpol_1, \bidpol_1, \B_1}\) and \(\t_2 = \zug{\actpol_1, \bidpol_2, \B_2}\), such that \(\actpol_i\) follows \(\actpol_i'\), and \(\bidpol_i'\) follows \(\bidpol_i'\) before reaching any BSCCs. 
	Because \(\comp{\t_1'}{\t_2'} \models T_1 \cap T_2\), \(\comp{\t_1}{\t_2}\) reach a BSCC \(\S\) such that \(\S \cap B_i \neq \emptyset\) for \(i = \{1, 2\}\).
	Within \(\S\), we let \(\t_i\) follow \PO winning strategy for the bidding game \(\S, \buchi(\buchiset_i)\), which exists because the threshold budget for \buchiword	objective is 0. 
	Thus, \(\comp{\t_1}{\t_2} \models \spec_1 \cap \spec_2\).  \qed
\end{proof}

Like reachability (see Rem.~\ref{rem:AA:reach:non-binary}), for B\"uchi objectives the same algorithm for $\aaprob$ for binary graphs can be used to obtain a sound solution for general graphs, and the completeness question is left open for future work.


\section{Assume-Guarantee Decentralized Synthesis}
\label{sec:assume-guarantee}

We present the assume-guarantee decentralized synthesis problem, the one with the highest degree of synchronization among the tenders, with the benefit of the most applicability.
In this synthesis procedure, we assume that we are given a pair of languages $\Ass_1,\Ass_2\subseteq \V^\omega$, called the \emph{assumptions}.
Intuitively, each tender $\t_i$ can \emph{assume} $\Ass_i$ is fulfilled by the other tender, and, in return, needs to \emph{guarantee} that $\Ass_{-i}$ is fulfilled, in addition to fulfilling own objective.

\begin{definition}[Contract-abiding tenders]\label{def:contract satisfaction}
	Let $\G$ be a graph, $\spec_i$ be an $\omega$-regular objective, and $\Ass_1,\Ass_2$ be a pair of $\omega$-regular assumptions in $\G$.
	We say a tender $\t_i=\tup{\actpol_i,\cdot,\cdot}\in \Tenders^\G$ fulfills $\spec_i$ under the contract $\tup{\Ass_1,\Ass_2}$, written $\t_i\models \AG{\Ass_i}{\spec_i}{\Ass_{-i}}$, iff 
	\begin{enumerate}[(a)]
		\item for every finite path $\rho$, if $\rho$ is in $\pref(\Ass_i)$, then $\rho\cdot\actpol_i(\rho)\in \pref(\Ass_1\cap\Ass_2)$, and
		\item for every other compatible tender $\t_{-i}\in\Tenders^\G$, we have $\comp{\t_i}{\t_{-i}}\models \safe_\G(\pref(\Ass_1\cap\Ass_2))\implies \left(\Ass_{-i}\land \left( \Ass_i\implies \spec_i \right)\right)$.
	\end{enumerate}
\end{definition}

Here, each tender $\t_i$ only make safety assumption on the other tender (Cond.~(a)), namely that the path does not leave the safe set $\pref(\Ass_i)$, and in return, provides full guarantee on $\Ass_{-i}$ (Cond.~(b)).
Normally, safety assumptions are not enough for fulfilling liveness guarantees and objectives~\cite{amla2001assume}.
But in bidding games, within the safe set, the players can use a known bidding tactic~\cite{AHC17} to accumulate enough budgets from time to time to reach the liveness goals always eventually.
We use $\Ass_1,\Ass_2$ as $\omega$-regular sets, though we conjecture that safety assumptions suffice.
The assume-guarantee distributed synthesis problem asks to compute a pair of tenders that fulfill their objectives under the given contract, as stated below.

\begin{myproblem}[$\agprob$]\label{prob:ag synthesis}
	Define $\agprob$ as the problem that takes as input a tuple $\tup{\G,\spec_1,\spec_2,\Ass_1,\Ass_2}$, 
	where $\G$ is a graph,
	$\spec_1$ and $\spec_2$ are overlapping objectives in $\G$,
	and $\Ass_1$ and $\Ass_2$ are two $\omega$-regular languages over $\V$ with $\vinit\in\pref(\Ass_1\cap\Ass_2)$,
	and the goal is to decide whether there exists a pair of tenders $\t_1 ,\t_2  \in \Tenders^\G$ such that:
		\begin{enumerate}[(I)]
			\item $\t_1$ and $\t_2$ are compatible, 
			\item $\t_1\models \AG{\Ass_1}{\spec_1}{\Ass_2}$, and
			\item $\t_2\models\AG{\Ass_2}{\spec_2}{\Ass_1}$. 
		\end{enumerate}
\end{myproblem}

When the assumptions allow all behaviors, i.e., $\Ass_1 =\Ass_2 = \V^\omega$, then $\agprob$ is equivalent to $\strongprob$.
On the other hand, when the assumptions allow only each other's objectives, i.e., $\Ass_1=\spec_1$ and $\Ass_2=\spec_2$, then we obtain a purely cooperative synthesis algorithm.
We prove that satisfaction of the contracts by a pair of tenders will imply satisfaction of $\spec_1\cap\spec_2$.

\begin{restatable}[Sound composition of contract-abiding tenders]{proposition}{assumeGuaranteeProofSound}
	\label{prop:assume-guarantee:proof rule soundness}
	Let $\t_1$ and $\t_2$ be tenders that fulfill the requirements stated in Prob.~\ref{prob:ag synthesis}.
	Then, $\comp{\t_1}{\t_2}\models \spec_1\cap\spec_2$.
\end{restatable}

\begin{proof}
	In the following, for a given language $L\in \V^\omega$, we write $\safe_\G(\pref(L))$ to denote the set of infinite paths which can always be extended to $L$, i.e., \linebreak
	$\set{\vinit v^1\ldots \in \infpaths(\T) \mid \forall i\geq 0\;.\; \vinit\ldots v^{i}\in \pref(L)}$.
	
	We claim that both assumptions $\Ass_1,\Ass_2$ will be fulfilled, from which Cond.~(b) of Def.~\ref{def:contract satisfaction} will imply satisfaction of both $\spec_1$ and $\spec_2$ by $\t_1$ and $\t_2$, respectively.
	Let $\Ass=\Ass_1\cap \Ass_2$, and $\Ass$ can be decomposed into safety and liveness components as $\Ass = \safe_\G(\pref(\Ass))\cap \left(\safe_\G(\pref(\Ass))\implies \Ass\right)$.
	We prove the claim on the two components separately.
	Firstly, the fact that $\comp{\t_1}{\t_2}$ implements $\pref(\Ass)$ on $\T$ can be proven by induction over the length of the generated path:
	The base case is given by the assumption $\vinit\in \pref(\Ass_1\cap\Ass_2)$, and for every finite path $\rho$, if $\t_i$ wins the bidding and if $\rho\in \pref(\Ass_1\cap\Ass_2)\subset \pref(\Ass_{i})$ then $\t_i$ needs to ensure that the next vertex $v'$ satisfies $\rho v'\in\pref(\Ass_i\cap\Ass_{-i})$ (consequence of Cond.~II-III of Prob.~\ref{prob:ag synthesis} and Cond.~(a) of Def.~\ref{def:contract satisfaction}), thereby implying that the path will always remain inside $\pref(\Ass_1\cap\Ass_2)$, proving the safety part.
	
	For the liveness part, we use known results from Richman bidding games, which guarantee that in an infinite horizon game, with any arbitrary positive initial budget, players can always eventually visit any vertex that can be reached \cite{LLPSU99}.
	This implies that if the invariance $\safe_\G(\pref(\Ass_1\cap\Ass_2))$ holds, then each tender $\t_i$ can actually fulfill $\Ass_{-i}$ (they are required to do so by Cond.~(b) of Def.~\ref{def:contract satisfaction}) when composed with any compatible tender in the long run.
	Therefore, $\Ass_1\cap\Ass_2$ will be fulfilled.
	\qed
\end{proof}

In bidding games literature, it is unknown how to compute strategies for objectives which can be violated if a given assumption is violated by the opponent, like in Cond.~(a) in Def.~\ref{def:contract satisfaction}.
The challenge stems from the lack of separation of the set of available actions to the players, preventing us to impose assumptions only on the opponent's behavior.
We present a practically motivated sound, but possibly incomplete, solution for the decentralized synthesis problem, by using a stronger way of satisfying the contract, namely by requiring each tender $\t_i$ to use actions so that the generated path remains in $\pref(\Ass_1\cap\Ass_2)$ all the time.
Formally, we say that the tender $\t_i$ \emph{strongly} fulfills $\spec_i$ under the contract $\tup{\Ass_1,\Ass_2}$, written $\t_i\models_{s} \AG{\Ass_i}{\spec_i}{\Ass_{-i}}$, if, instead of Cond.~(a) of Def.~\ref{def:contract satisfaction}, for every finite path $\rho$, we have $\rho\cdot\actpol_i(\rho)\in\pref(\Ass_1\cap\Ass_2)$, regardless of whether $\rho\in\pref(\Ass_i)$ or not, and moreover Cond.~(b) of Def.~\ref{def:contract satisfaction} is fulfilled.
It is easy to show that $\t_i\models_s \AG{\Ass_i}{\spec_i}{\Ass_{-i}}$ implies $\t_i\models \AG{\Ass_i}{\spec_i}{\Ass_{-i}}$, so that $\spec_1\cap\spec_2$ will be fulfilled.

Similar to $\aaprob$, we extract a sub-graph $\G'$ of $\G$, called the \emph{largest contract-satisfying sub-graph}, whose every path belongs to $\pref(\Ass_1\cap\Ass_2)$, and vice versa; we omit the construction, which follows usual automata-theoretic procedure from the literature~\cite{alpern1987recognizing}.
For example, in Ex.~\ref{ex:motivation:assume-guarantee}, the largest contract-satisfying sub-graph of the graph in Fig.~\ref{fig:motivating example:contract} is the one that only excludes the vertices $c$ and $f$.
It follows that when the tenders strongly fulfill their objectives under the contracts, it is guaranteed that every path always remains in $\G'$.

\begin{restatable}[Assume-guarantee decentralized synthesis]{theorem}{AssumeGuaranteeDecision}
	\label{thm:assume-guarantee decision procedure}
	Let $\G=\tup{\V,\vinit,\E}$ be a graph, $\spec_1$ and $\spec_2$ be a pair of overlapping $\omega$-regular objectives, and $\Ass_1$ and $\Ass_2$ be $\omega$-regular assumptions.
	Let $\G'$ be the largest contract-satisfying sub-graph of $\G$.
	A pair of robust tenders exist if $\th_{\Ass_2\cap\spec_1}^{\G'}(\vinit) + \th_{\Ass_1\cap\spec_2}^{\G'}(\vinit) < 1$.
	Moreover, $\agprob$ is in PTIME.
\end{restatable}

\begin{proof}
	Let $\G,\spec_1,\spec_2,\Ass_1$, and $\Ass_2$ are as in the theorem statement.
	We construct a new graph $\G'$, which is a sub-graph of $\G$, whose every path belongs to $\pref(\Ass_1\cap\Ass_2)$ in $\G$, and vice versa; the construction follows usual automata-theoretic procedure from the literature~\cite{alpern1987recognizing}.
	We assume that tenders strongly fulfill their objectives under the contracts, which guarantees that every path will remain inside $\G'$ all the time.
	
	We now solve the strong decentralized synthesis problems on $\G'$ with objectives $\spec_1' = \Ass_2\cap \spec_1$ and $\spec_2' = \Ass_1\cap\spec_2$, and if $\tup{\G',\spec_1',\spec_2'}\in \strongprob$, we claim that $\tup{\G',\spec_1,\spec_2,\Ass_1,\Ass_2}\in \agprob$, and the certificates $\t_1,\t_2$ of $\strongprob$ are also certificates of $\agprob$.
	The claim follows by observing that the tenders are compatible, and moreover each tender $\t_i$ always maintains that the path remains inside $\pref(\Ass_1\cap\Ass_2)$ (i.e., $\G'$), and moreover it fulfills $\spec_i'$, thereby satisfying $\t_i\models_s\AG{\Ass_i}{\spec_i}{\Ass_{-i}}$, which implies the satisfaction of $\t_i\models\AG{\Ass_i}{\spec_i}{\Ass_{-i}}$.
	\qed
\end{proof}


\section{Related Work}

\emph{Shielding}~\cite{KAB+17} is a framework in which a runtime monitor called a {\em shield} enforces an unverified policy $\pi$ (e.g., generated using reinforcement learning~\cite{AB+19}) to satisfy a given specification. A shield operates by observing, at each point in time, the action proposed by $\pi$ and can alter it, e.g., if safety is violated.  
The choice of who acts at each point in time, $\pi$ or the shield, can be seen as a scheduling choice similar to our setting. 
However,  the goals of the two approaches are different: our goal is to design tenders for modular policy synthesis, whereas a shield is meant as a verified ``wrapper'' for a complex policy. 
Technically, in auction-based scheduling, the scheduling depends on the auction which is external to the policies, whereas in shielding, it is the shield who chooses whether to override $\pi$. 

In {\em distributed reactive synthesis}~\cite{pneuli1990distributed,kupermann2001synthesizing,finkbeiner2005uniform}, the goal is to design a collection of Mealy machines whose communication is dictated by a given {\em communication architecture}. 
Distributed synthesis is well studied and we point to a number of works that considered objectives that are a conjunction $\spec_1\land\spec_2\land \ldots$ of sub-objectives $\spec_1,\spec_2,\ldots$~\cite{filiot2010compositional,basset2018compositional,majumdar2020assume,finkbeiner2022compositional,bansal2022compositional,anand2023contract}. 
While there is a conceptual similarity between our synthesis of tenders and the synthesis of Mealy machines, there is a fundamental difference between the approaches.
Namely, our composition is based on scheduling, i.e., exactly one policy is scheduled at each point in time, whereas in distributed synthesis, the composition of the Mealy machines is performed in parallel, i.e., they all read and write at each point in time.

Our algorithms build upon the rich literature on bidding games on graphs. 
The bidding mechanism that we focus on is called {\em Richman} bidding~\cite{LLPU96,LLPSU99,AHC19}. Other bidding mechanisms have been studied: {\em poorman}~\cite{AHI18}, {\em taxman}~\cite{AHZ21}, and {\em all-pay}~\cite{AIT20,AJZ21}. Auction-based scheduling can be instantiated with any of these mechanisms and the properties from bidding games transfer immediately (which differ significantly for quantitative objectives). 
Of particular interest in practice is {\em discrete bidding}, in which the granularity of the bids is restricted~\cite{DP10,AAH21,AS22}. 
To the best of our knowledge, beyond our work, non-zero-sum bidding games have only been considered in~\cite{meir2018bidding}. 
The solution concept that they consider is {\em subgame perfect equilibrium} (SPE). While it is suitable to model the interaction between selfish agents, we argue that it is less suitable in decentralized synthesis. 

There are many works on designing optimal policies for multi-objective sequential decision making problems for various different system models; see the
survey by Roijers et al.~\cite{roijers2013survey} and works on multi-objective stochastic \linebreak games ~\cite{basset2015strategy,basset2018compositional,chatterjee2018combinations}.
To the best of our knowledge, no prior work considers the decomposition of the problem into individual task-dependent policies like us.
Auctions to distribute tasks to agents have been considered\linebreak extensively~\cite{farber1972structure,smith1980contract,chong2003sensor,dias2006market,basile2019auction,de2020automated,ouelhadj2009survey}. 
Their goal is very different: their agents bid for tasks, that is, a bid represents an agent's cost (e.g., in terms of resources) for performing a task. 
The auction then allocates the tasks to agents so as to minimize the individual costs, giving rise to an efficient global policy.

	\section{Conclusion and Future Work}
	We present the auction-based scheduling framework. Rather than synthesizing a monolithic policy for a conjunction of objectives $\spec_1\land\spec_2$, we synthesize two independent tenders for each of the objectives and compose the tenders at runtime using auction-based scheduling. A key advantage of the framework is modularity; each tender can be synthesized and modified independently. 
	We study three instantiations of decentralized synthesis in planning problems with varying degree of flexibility and practical usability, and develop algorithms based on bidding games. 
	Interestingly, we show that a pair of admissible-winning tenders always exists in binary graphs for reachability objectives and they can be found in PTIME. 
	This positive result illustrates the strength and potential of the auction-based scheduling framework. 
		
	There are plenty of directions of future research and we list a handful.  
	First, we consider only qualitative objectives and it is interesting to lift the results to \emph{quantitative} objectives, where one can quantify the fairness achieved by the scheduling mechanism in a fine-grained manner. Moreover, it is appealing to employ the rich literature on {\em mean-payoff} bidding games. 
	Second, we consider a conjunction of two objectives, and it is interesting to extend the approach to a conjunction of multiple objectives. 
	This will require extending the theory of bidding games to the multi-player setting, which have not yet been studied. 
	Finally, it is particularly interesting to extend the technique of auction-based scheduling beyond path-planning problems, for example, it is interesting to consider decentralized synthesis of controllers that operate in an adversarial or probabilistic environment. Again, the corresponding bidding games need to be studied (so far only sure winning has been considered for bidding games played on MDPs~\cite{AHIN19}). 
	
	\subsubsection*{Acknowledgements}	
	    This work was supported in part by the ERC project ERC-2020-AdG 101020093 and by ISF grant no. 1679/21.

	\bibliographystyle{splncs04}
	\bibliography{references,ga}
	
\end{document}